\documentclass[conference]{IEEEtran}
\IEEEoverridecommandlockouts
\usepackage{cite}
\usepackage{amsmath,amssymb,amsfonts}
\usepackage{algpseudocode}
\usepackage{graphicx}
\usepackage{textcomp}
\usepackage{xcolor}
\usepackage{multirow}
\usepackage{fontawesome}

\usepackage[mathscr]{euscript}
\usepackage[linesnumbered,ruled,vlined]{algorithm2e}
\usepackage{amsthm}
\theoremstyle{definition}
\newtheorem{definition}{Definition}
\newtheorem{problem}{Problem}

\usepackage{bm}
\usepackage{booktabs}
\usepackage{stmaryrd}
\usepackage{comment}
\usepackage[hidelinks]{hyperref}
\usepackage{microtype}
\usepackage{epsfig} % EPS figure
\ifCLASSOPTIONcompsoc
	\usepackage[caption=false,font=normalsize,labelfon
t=sf,textfont=sf]{subfig}
\else
	\usepackage[caption=false,font=footnotesize]{subfig}
\fi

\usepackage[export]{adjustbox}
\usepackage{fixltx2e} % subscript
\usepackage{enumitem} % for labeling enumerators

\newtheorem{thm}{Theorem}
\newtheorem{obs}{Observation}

\newcommand{\mat}[1]{\bm{#1}}
\newcommand{\mrow}[2]{\bm{#1}(#2, :)}
\newcommand{\mcol}[2]{\bm{#1}(:, #2)}
\newcommand{\mt}[1]{{#1}'}
\newcommand{\mi}[1]{{#1}^\dagger}

\newcommand{\matri}[2]{{#1}_{(#2)}}
\newcommand{\order}{M}
\newcommand{\modeSize}[1]{N_{#1}}

\newcommand{\reconst}[1]{\widetilde{#1}}
\newcommand{\fm}[2]{{#1^{(#2)}}}
\newcommand{\fmInF}[2]{\bm{#1}^{(#2)}}
\newcommand{\rank}{R}

\newcommand{\ltwotext}{\ell^{2}}

\newcommand{\khatri}{\odot}
\newcommand{\hada}{\ast}
\newcommand{\multiKha}[3]{\odot_{m \neq #2}^{#3} \bm{#1}^{(m)}}
\newcommand{\multiKhaEq}[3]{\odot_{m = #2}^{#3} \bm{#1}^{(m)}}
\newcommand{\multiKhaN}[3]{\odot_{n \neq #2}^{#3} \bm{#1}^{(n)}}

\newcommand{\ata}[2]{\mt{\fm{#1}{#2}} \fm{#1}{#2}}
\newcommand{\multiHada}[3]{\ast_{m \neq #2}^{#3} \ata{#1}{m}}
\newcommand{\multiHadaEq}[3]{\ast_{m = #2}^{#3} \ata{#1}{m}}

\newcommand{\multiHadaN}[3]{\ast_{n \neq #2}^{#3} \ata{#1}{n}}
\newcommand{\multiHadaNCustom}[3]{\ast_{n \neq #2}^{#3} #1}

\newcommand{\dims}[2]{\mathbb{R}^{#1 \times #2}}
\newcommand{\threeDims}[3]{\mathbb{R}^{#1 \times #2 \times \cdots \times #3}}
\newcommand{\threeDimsBack}[3]{\mathbb{R}^{#1 \times \cdots \times #2 \times #3}}
\newcommand{\twoDims}[2]{\mathbb{R}^{#1 \times \cdots \times #2}}
\newcommand{\cpdWord}{CPD\xspace} 
\newcommand{\cpd}[1]{\llbracket #1 \rrbracket}

\newcommand{\tensorInF}[1]{\bm{\mathcal{#1}}}

\newcommand{\sumP}[2]{\sum\nolimits_{#1=1}^{#2}}

\newcommand{\ltwosmall}[1]{\|#1\|_{2}}

\newcommand{\lfro}[1]{\Big\|#1\Big\|_{F}}
\newcommand{\lfrosmall}[1]{\|#1\|_{F}} 
\newcommand{\lfrof}{\|\cdot\|_{F}} 

\newcommand{\batch}{T}

\newcommand{\tLen}{W}
\newcommand{\tensorWindow}[2]{\tensorInF{#1}\left(#2, \tLen\right)}

\newcommand{\matset}[2]{\{#1\}_{m=1}^{#2}}
\newcommand{\degree}[1]{deg(#1)}

\newcommand{\method}{\textsc{SliceNStitch}\xspace}

\newcommand{\wals}{\textsc{SNS\textsubscript{mat}}\xspace}
\newcommand{\selals}{\textsc{SNS\textsubscript{vec}}\xspace}
\newcommand{\selccd}{\textsc{SNS\textsuperscript{+}\textsubscript{vec}}\xspace}
\newcommand{\hyals}{\textsc{SNS\textsubscript{rnd}}\xspace}
\newcommand{\hyccd}{\textsc{SNS\textsuperscript{+}\textsubscript{rnd}}\xspace}

\newcommand{\walslong}{\textsc{SliceNStitch-Matrix}\xspace}
\newcommand{\selalslong}{\textsc{SliceNStitch-Vector}\xspace}
\newcommand{\hyalslong}{\textsc{SliceNStitch-Random}\xspace}
\newcommand{\ccdlong}{\textsc{SliceNStitch-Stable}\xspace}

\newcommand\blue[1]{\textcolor{black}{#1}}
\definecolor{forestgreen}{rgb}{0.13, 0.55, 0.13}
\newcommand\green[1]{\textcolor{black}{#1}}
\newcommand{\smallsection}[1]{{\vspace{0.05in} \noindent {\bf{\underline{\smash{#1}:}}}}}

\newcommand{\cp}{CANDECOMP/PARAFAC\xspace}
\newcommand{\tuplen}{e_{n}{=}(i_{1},\cdot\cdot\cdot,i_{\order-1},v_n)}
\newcommand{\pairn}{\left(\tuplen,t_n\right)}
\newcommand{\pairnshort}{\left(e_n,t_n\right)}

\newcommand{\entryi}{x_{i_{1},i_{2},\cdots,i_{\order}}}
\newcommand{\indexi}{(i_{1},i_{2},\cdots,i_{\order})}

\newcommand{\indeximinus}{(i_{1},\cdots,i_{\order-1})}
\newcommand{\indexiW}{(i_{1},\cdots,i_{\order-1},W)}
\newcommand{\entryiW}{x_{i_{1},\cdots,i_{\order-1},W}}
\newcommand{\indexiw}{(i_{1},\cdots,i_{\order-1},w)}

\newcommand{\entryiWw}{x_{i_{1},\cdots,i_{\order-1},W-w}}

\newcommand{\entryiWwo}{x_{i_{1},\cdots,i_{\order-1},W-w+1}}
\newcommand{\entryione}{x_{i_{1},\cdots,i_{\order-1},1}}
\newcommand{\indexj}{j_{1},\cdots,j_{\order}}
\newcommand{\entryjshort}{x_{J}}
\newcommand{\entryjbarshort}{\bar{x}_{J}}
\newcommand{\entryjtildeshort}{\tilde{x}_{J}}
\newcommand{\indexiParam}[1]{i_{1},\cdots,i_{\order-1},#1}

\newcommand{\tabbad}[1]{{#1 (\faThumbsODown)}}
\newcommand{\tabgood}[1]{{#1 (\faThumbsOUp)}}

\newcommand{\concat}{||}

\newcommand{\thre}{\theta}
\newcommand{\clipThre}{\eta}

\SetCommentSty{mycommfont}
\SetKwComment{rCom}{$\triangleright$\ }{}

\def\BibTeX{{\rm B\kern-.05em{\sc i\kern-.025em b}\kern-.08em
    T\kern-.1667em\lower.7ex\hbox{E}\kern-.125emX}}
\makeatletter
\def\endthebibliography{%
	\def\@noitemerr{\@latex@warning{Empty `thebibliography' environment}}%
	\endlist
}
\makeatother
\begin{document}

\title{SliceNStitch: Continuous CP Decomposition of \blue{Sparse} Tensor Streams}

\author{
\IEEEauthorblockN{Taehyung Kwon\IEEEauthorrefmark{1}\IEEEauthorrefmark{2}, Inkyu Park\IEEEauthorrefmark{1}\IEEEauthorrefmark{2}, Dongjin Lee\IEEEauthorrefmark{4}, and Kijung Shin\IEEEauthorrefmark{2}\IEEEauthorrefmark{4}}
\IEEEauthorblockA{\IEEEauthorrefmark{2}Graduate School of AI and \IEEEauthorrefmark{4}School of Electrical Engineering, KAIST, Daejeon, South Korea
\\\{taehyung.kwon, inkyupark, dongjin.lee, kijungs\}@kaist.ac.kr}
}

% \author{\IEEEauthorblockN{Taehyung Kwon}
% \and
% \IEEEauthorblockN{Inkyu Park}
% \and
% \IEEEauthorblockN{Dongjin Lee}
% \and
% \IEEEauthorblockN{Kijung Shin}
% \and
% }
%\setlength{\baselineskip}{13pt}
%\input{cover_revision.tex}
%\setlength{\baselineskip}{12pt}

\maketitle

\begingroup\renewcommand\thefootnote{*}
\footnotetext{Equal Contribution}
\endgroup

\begin{abstract}
	Consider traffic data (i.e., triplets in the form of source-destination-timestamp) that grow over time.
Tensors (i.e., multi-dimensional arrays) with a time mode are widely used for modeling and analyzing such multi-aspect data streams.
In such tensors, however, new entries are added only once per period, which is often an hour, a day, or even a year.
This discreteness of tensors has limited their usage for real-time applications, where new data should be analyzed instantly as it arrives.

How can we analyze time-evolving multi-aspect \green{sparse} data `continuously' using tensors where time is `discrete'?
We propose \method for continuous CANDECOMP/PARAFAC (CP)
decomposition, which has numerous time-critical applications, including anomaly detection, recommender systems, and stock market prediction.
\method changes the starting point of each period adaptively, based on the current time, and updates factor matrices (i.e., outputs of CP decomposition) instantly as new data arrives.
We show, theoretically and experimentally, that \method is (1) \textit{`Any time'}: updating factor matrices immediately without having to wait until the current time period ends, (2) \textit{Fast}: with constant-time updates up to \blue{$464\times$} faster than online methods, and (3) \textit{Accurate}: with \blue{fitness} comparable (specifically, $72-100\%$) to offline methods.

\end{abstract}

%not necessary
%\begin{IEEEkeywords}
%component, formatting, style, styling, insert
%\end{IEEEkeywords}

\section{Introduction}
\label{sec:intro}
%\red{P1. What is a tensor and why is it important?}
%Tensor ...
%It is a expressive tool widely used for representing various data ...
%Excamples of data that can be represented as tensors...
%There are a lot of tools available for analysis.. 
Tensors (i.e., multidimensional arrays) are simple but powerful tools widely used for representing time-evolving multi-aspect data from various fields, including bioinformatics \cite{zhao2005tricluster}, data mining \cite{koutra2012tensorsplat,cai2015facets}, text mining \cite{bader2008discussion}, and cybersecurity \cite{bruns2016cyber, fanaee2016tensor}.
For example, consider a traffic dataset given as triplets in the form of (source, destination, timestamp).
The dataset is naturally represented as a $3$-mode tensor whose three modes correspond to sources, destinations, and time ranges, respectively (see Fig.~\ref{subfig:coarse} and \ref{subfig:fine}).
Each $(i,j,k)$-th entry of the tensor represents the amount of traffic from the $i$-th source to the $j$-th destination during the $k$-th time range.
%For example, a traffic \textit{(Central Park, JFK Airport, 3 p.m., 1)} contributes to the entry in $(i, j, k)$ of a tensor as much as $1$ if Central Park, JFK Airport, and 3 p.m. are mapped to $i$, $j$, and $k$ repsectively.
%Then, such tensor is used to express the amounts of traffics.

\begin{figure}[t]
    \vspace{-2.5mm}
    \centering
    \subfloat[Coarse-grained Tensor]{\includegraphics[width=0.417\linewidth]{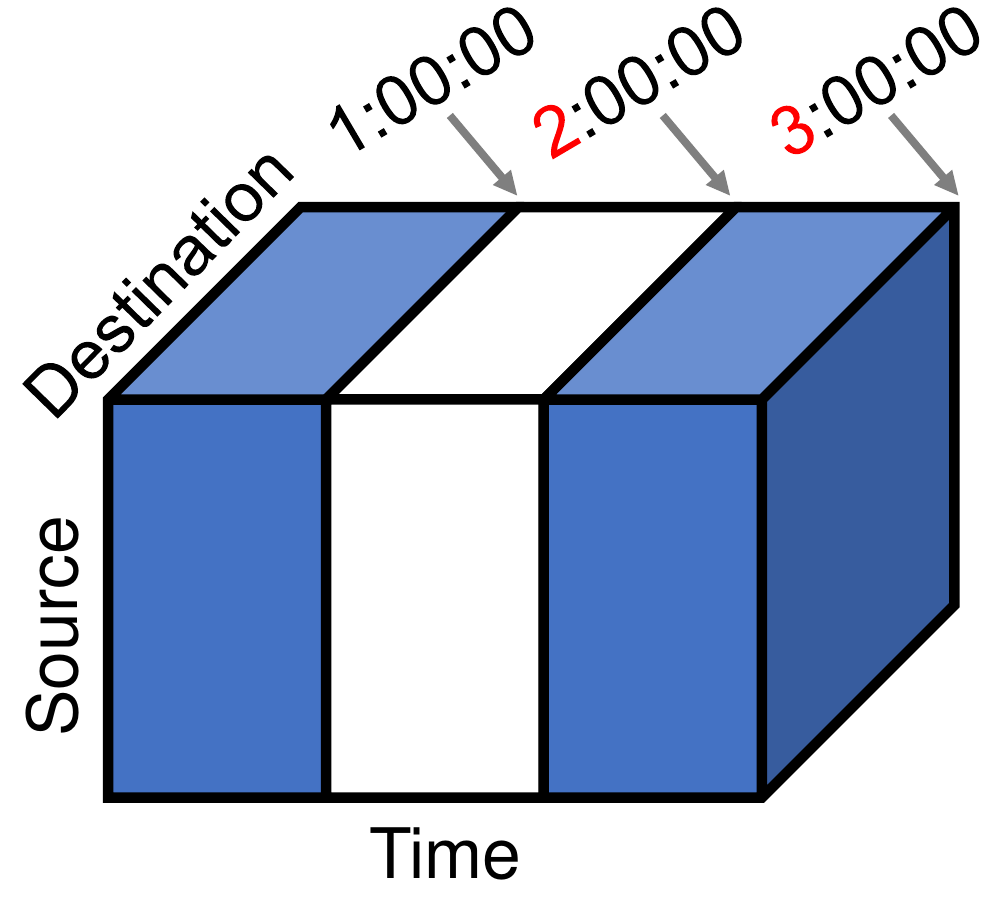}\label{subfig:coarse}}
    \subfloat[Fine-grained Tensor]{\includegraphics[width=0.4155\linewidth]{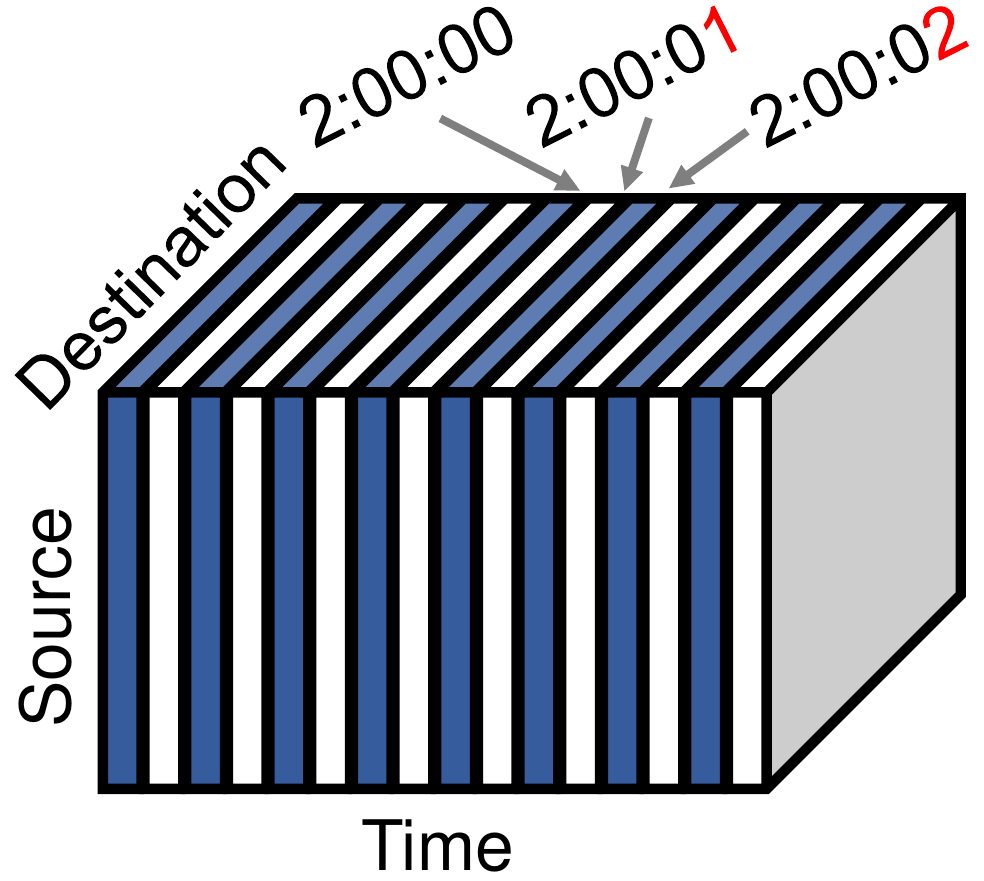}\label{subfig:fine}} \\ \vspace{-3mm}
    \subfloat{\includegraphics[width= 0.85\linewidth]{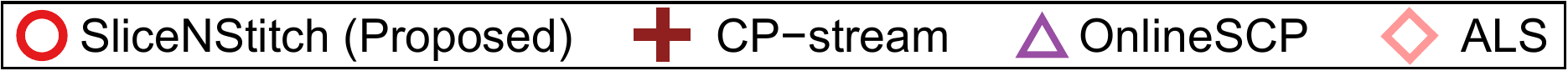}} \\ \vspace{-3mm} \addtocounter{subfigure}{-1}
    \subfloat[\blue{Average Fitness}]{\includegraphics[width=0.33\linewidth]{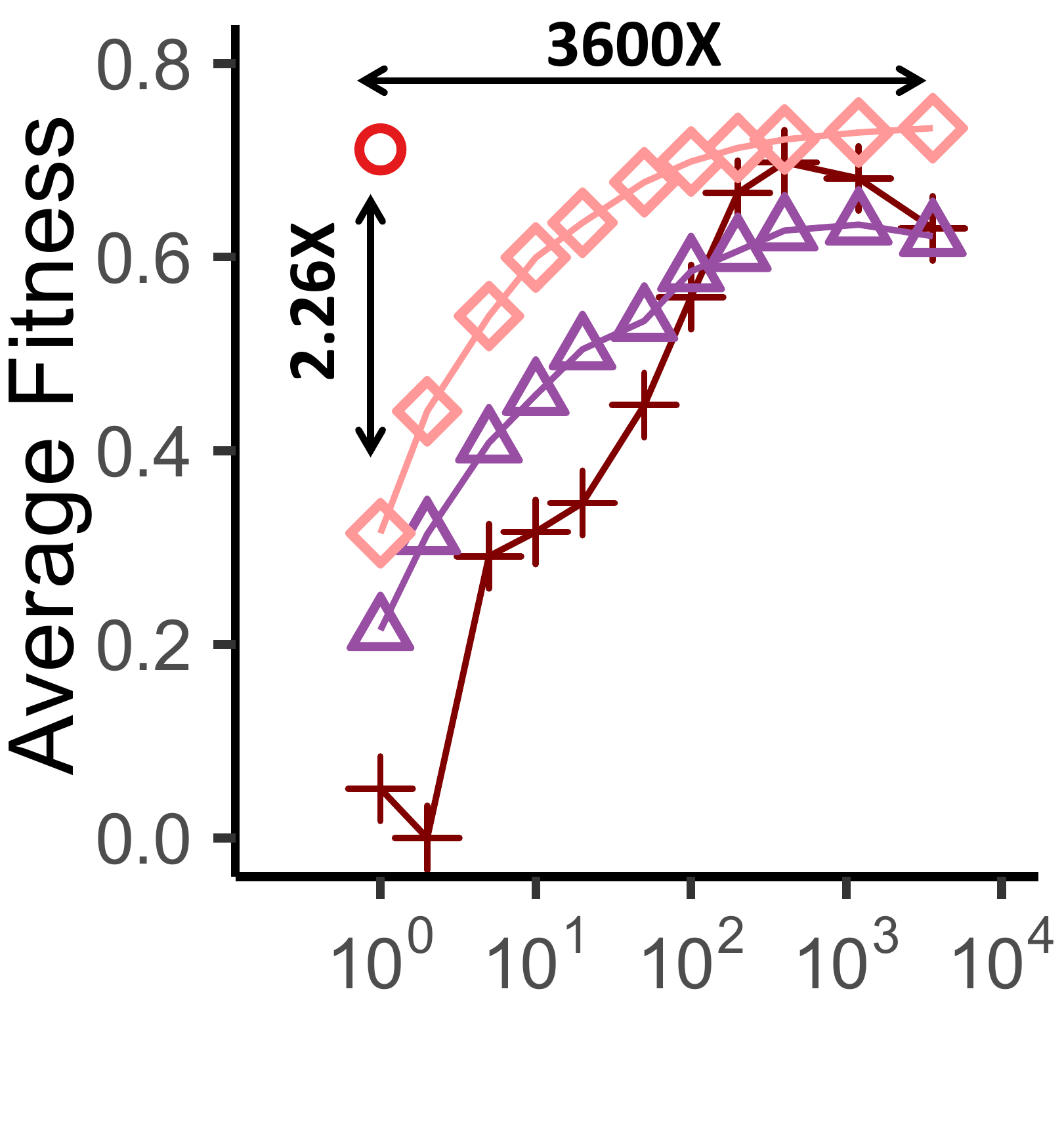}\label{subfig:simple_cp:fitness}}
    \subfloat[\green{Number of Parameters}]{\includegraphics[width=0.33\linewidth]{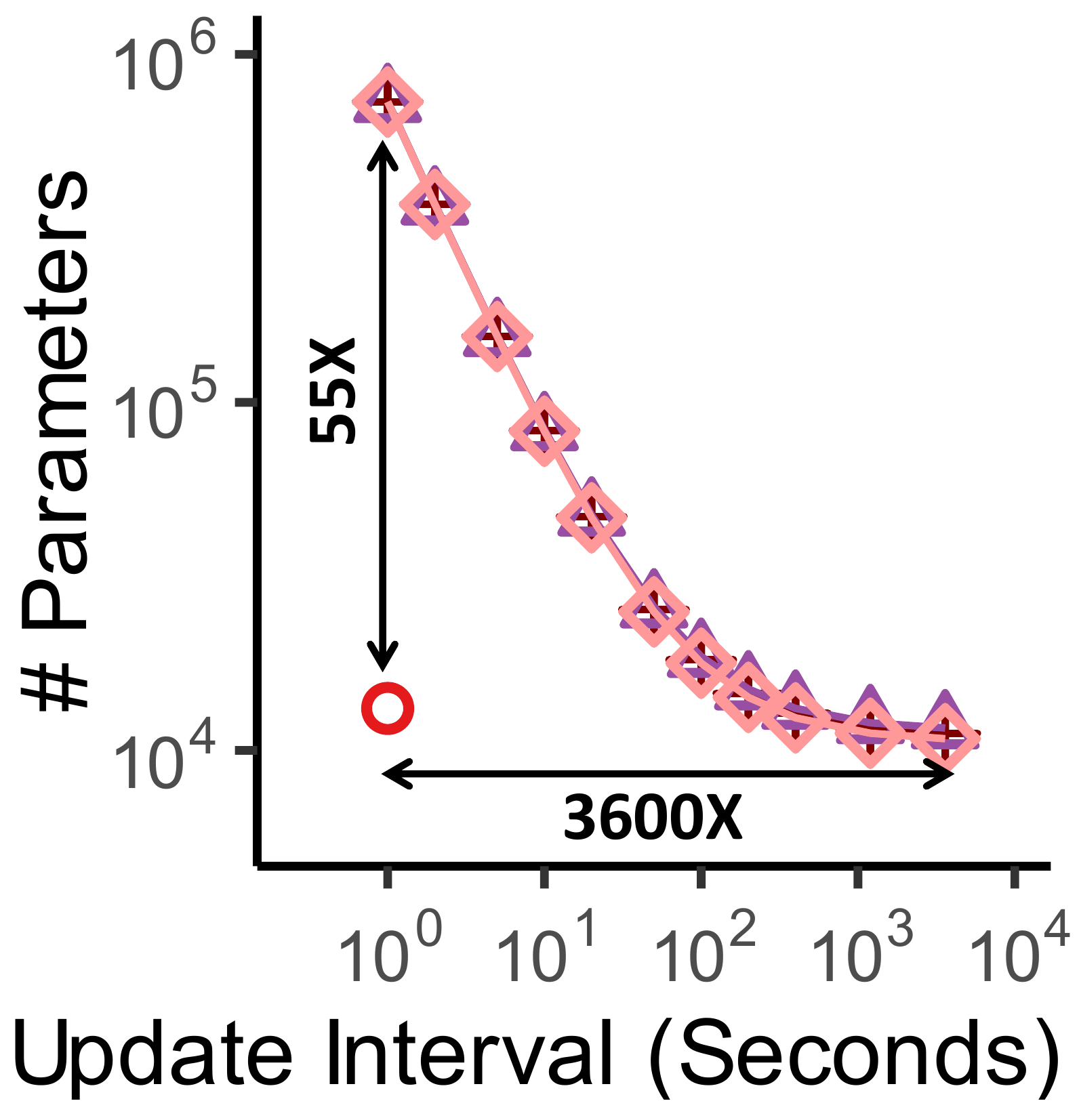}\label{subfig:simple_cp:param}}
    \subfloat[\blue{Runtime per Update}]{\includegraphics[width=0.33\linewidth]{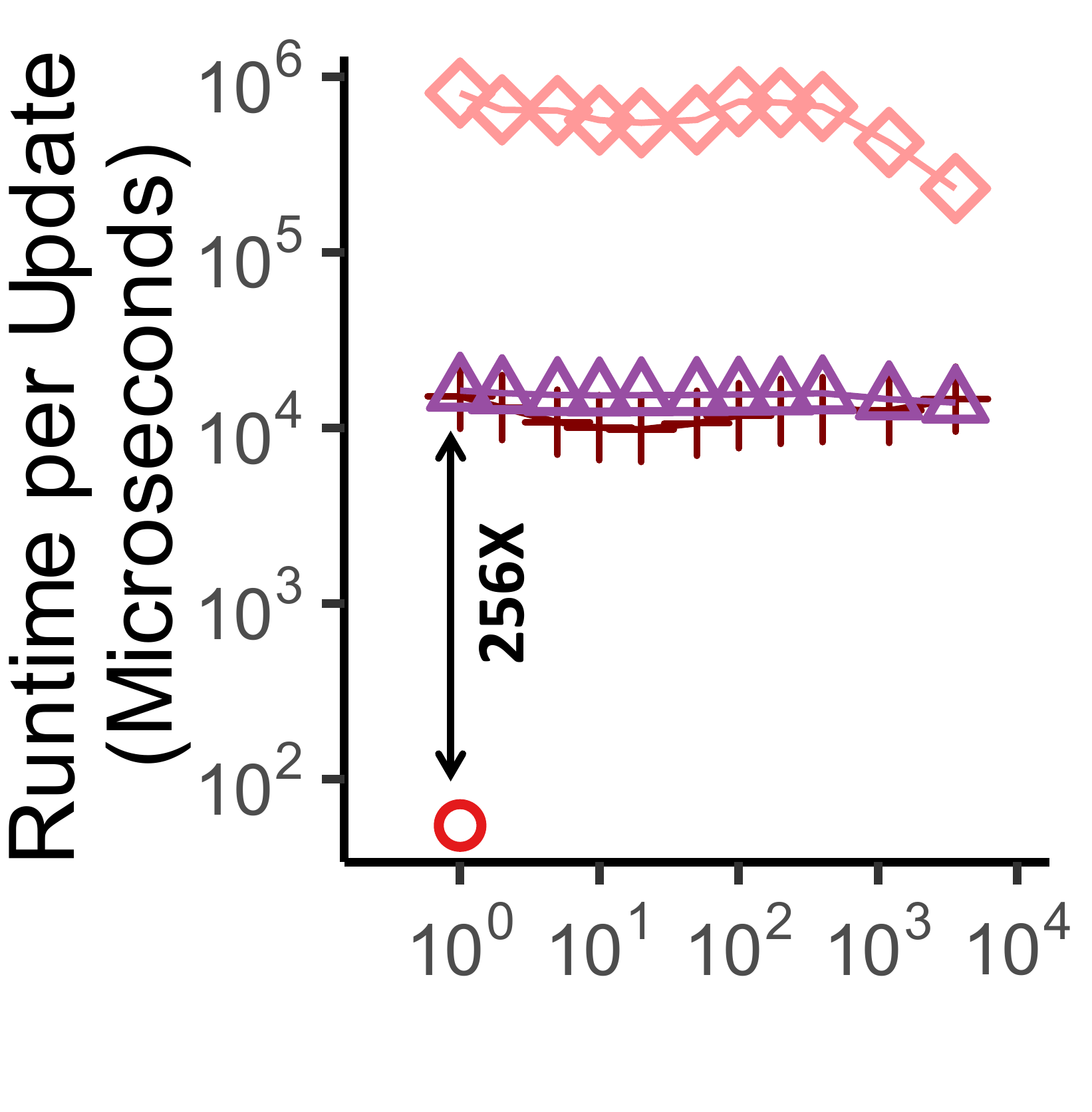}\label{subfig:simple_cp:time}} \\ \vspace{-2mm}
    \subfloat[Summary: \method is fast, space-efficient, and accurate.]{
        \scalebox{0.74}{
            \begin{tabular}{l||c|c||c}
                \toprule
                                & \multicolumn{2}{c||}{\green{Methods Based on Conventional CPD}} & \textbf{\method}                            \\
                                & (Coarse-grained)                                         & (Fine-grained)   & \textbf{(Proposed)}      \\
                \midrule
                Update Interval & \tabbad{Long}                                            & \tabgood{Short}  & \textbf{\tabgood{Short}} \\
                Parameters      & \tabgood{Few}                                            & \tabbad{Many}    & \textbf{\tabgood{Few}}   \\
                \green{Fitness}        & \tabgood{High}                                           & \tabbad{Low}     & \textbf{\tabgood{High}}  \\
                \bottomrule
            \end{tabular}
        }
        \label{Tab:summary}
    }   \\
    \caption{ \label{fig:motivation}
    	\blue{
        \textbf{Advantages of \method.}
        \protect\subref{subfig:coarse} \protect\subref{subfig:fine} Coarse-grained and fine-grained tensors where update intervals are $1$ hour and $1$ second, respectively.
        \protect\subref{subfig:simple_cp:fitness} \protect\subref{subfig:simple_cp:param}        
		Given a tensor stream (see Section~\ref{sec:exp:continuous} for detailed experimental settings), \method updates factor matrices (i.e., outputs of CPD) instantly (i.e., with a short update interval) while achieving high fitness with a small number of parameters.
		\protect\subref{subfig:simple_cp:time} Even runtime per update is shorter in \method than in the three considered methods based on conventional CPD.
		\protect\subref{Tab:summary} A summary of the comparisons in \protect\subref{subfig:simple_cp:fitness} and \protect\subref{subfig:simple_cp:param}.
    }}
\end{figure}

%\red{P2. What is CP decomposition? Why does it matter?}
%Among many available tools.. CP decomposition is ..
%It decomposes the input tensor into...
%CP decomposition has been used for variuos applications ...
Once we represent data as a tensor, many tools \cite{hitchcock1927expression,carroll1970analysis,harshman1972parafac2,tucker1966some} are available for tensor analysis, and \cp decomposition (CPD) \cite{hitchcock1927expression} is one of the most widely used tools.
Given an $\order$-mode tensor, CPD gives a low-rank approximation, specifically a sum of few outer products of $\order$ vectors, which form $\order$ factor matrices. %, and alternating least squares (ALS) \cite{kolda2009tensor} is widely used for obtaining the factor matrices. %given a tensor and a rank.
CPD has been used for various applications, and many of them, including anomaly detection \cite{koutra2012tensorsplat}, recommendation \cite{yao2015context,shin2016fully}, stock market prediction \cite{spelta2017financial}, and weather forecast \cite{xu2018muscat}, are time-critical.
%and feature extraction \cite{cong2015tensor}.

While tensors and CPD are powerful tools, they are not suitable for real-time applications since
time in them advances in a discrete manner, specifically once per period.
For example, in the tensor in Fig.~\ref{subfig:coarse}, each slice represents the amounts of traffic for one hour, and thus the tensor grows with a new slice only once per hour.
That is, it may take one hour for new traffic to be applied to the tensor.
For instance, traffic occurring at 2:00:01 is applied to the tensor at 3:00:00.
Due to this discreteness, the outputs of CPD (i.e., factor matrices) are updated also only once per period even if incremental algorithms \cite{smith2018streaming, zhou2018online, gujral2018sambaten} are used.

% \cite{smith2018streaming, zhou2018online, gujral2018sambaten} even state-of-the-art incremental algorithms for CP decomposition  update outputs (i.e., factor matrices) only once per period. 

%\red{instead of listing fields, could you specify tasks, such as anomaly detection and recommendation}

%\red{P3. What is the limitation of tensors and CP Decomposition}
%time is discrete... updated once per time period ...
%CP decomposition algorithms, including so-called streaming or online algorithms (shou-zhou, sambatan, SMF, etc.), are also limited ..
%They cannot be used for real-time applications
%As a tensor discretely holds data, time components in the tensor are also given in discrete numbers with a fixed time interval. 
%So one should wait a given cycle to increase the tensor in the time dimension.
%Existing state-of-the-art solutions \cite{smith2018streaming, zhou2018online, gujral2018sambaten} for streaming tensor decompositions and factorizations are also applied once per time period.
%Consider a case where a tensor is composed of a 1-day unit, a tensor decomposition is updated at every midnight, and data is added in a 1-second unit. 
%One can not reflect the data collected until the middle of the day(e.g. 10:00 am or 6:00 pm) to the lastly updated tensor and its decomposition.

%\red{P4. Why a simple solution does not work (Fig.~1)}
How can we perform CPD `continuously' for real-time applications?
A potential solution is to make the granularity of the time mode extremely fine.
According to our preliminary studies, however, it causes the following problems:
\vspace{-0.5mm}
\begin{itemize}[leftmargin=*]
    \item \textbf{Degradation of \green{Fitness} (Fig.~\ref{subfig:simple_cp:fitness})}
          An extremely fine-grained time mode results in an extremely sparse tensor, which is known to be of high rank \cite{pasricha2019adaptive}, and thus it degrades the \green{fitness} of the low-rank approximation, such as conventional CPD. %, more difficult to be applied due to their poor performance on the fine-grained tensor.
          As shown in Fig.~\ref{subfig:simple_cp:fitness}, the finer the time mode is, the lower the fitness of CPD is.
   %       \change{Note that we sum the rows of all fine-grained time mode factor matrices of baselines so that one row corresponds to one hour.
    %          Using the fine-grained time mode factor matrix as it is has worse performances than the results of Fig.~\ref{subfig:simple_cp:fitness} (fitness less than 0.055 when the update interval is 1 second).}

          %The finer the granularity of the time mode is, the lower the accuracy of CP decomposition is.
          %This is basically from the property of the low-rank approximation. When the time interval decreases, each batch of the tensor becomes sparse, so that each non-zero entry is considered as a higher-rank component, such as noise. This makes CP decomposition more difficult to fit data using low-rank, which in turn reduces accuracy.
    \item \textbf{Increase in the Number of Parameters (Fig.~\ref{subfig:simple_cp:param}):}
          The parameters of CPD are the entries of factor matrices, as explained in Section~\ref{sec:prelim}, and the size of each factor matrix is proportional to the length of the corresponding mode of the input tensor.
          An extremely fine-grained time mode leads to an extremely long time mode and thus extremely many parameters, which require huge computational and storage resources.
          As shown in Fig.~\ref{subfig:simple_cp:param}, the finer the time mode is, the more parameters CPD requires.
          %When we construct a tensor by using the same amount of data, the size of the tensor must increase in the time dimension when we use a shorter time interval. Since the number of parameters is proportional to the size of the tensor, it requires more parameters to be used.
\end{itemize}
In this work, we propose \method for continuous CPD without increasing the number of parameters.
It consists of a data model and online algorithms for CPD.
From the \textbf{data model aspect}, we propose the continuous tensor model for time-evolving tensors. In the model, the starting point of each period changes adaptively, based on the current time, so that newly arrived data are applied to the input tensor instantly.
%Note that each period does not have to be extremely short.
From the \textbf{algorithmic aspect}, we propose a family of online algorithms for CPD of \green{sparse tensors} in the continuous tensor model.
They update factor matrices instantly in response to each change in an entry of the input tensor.
To the best of our knowledge, they are the first algorithms for this purpose, and existing online CPD algorithms \cite{smith2018streaming, zhou2018online, gujral2018sambaten} update factor matrices only once per period.
We summarize our contributions as follows:
\begin{itemize}[leftmargin=*]
    \item{\textbf{New data model}}: We propose the continuous tensor model, which allows for processing time-evolving tensors continuously in real-time for time-critical applications.
          % The model does not suffer from extreme increase in rank of the input tensor or the number of parameters.
          %     a new data model called for continuously analyzing tensor streams.
    \item{\textbf{Fast online algorithms}}: We propose the first online algorithms that update outputs of CPD instantly in response to changes in an entry.
          Their \green{fitness} is comparable (specifically, $72-100\%$) even to offline competitors, and
          an update by them is up to \blue{$464\times$} faster than that of online competitors. %, which update outputs only once per period. 

          %    than directly applying the ALS algorithm. 
          %    We also devise algorithms with constant update time (Theorems \ref{thm:hyals} and \ref{thm:hyccd})
    \item{\textbf{Extensive experiments}}: We extensively evaluate our algorithms on 4 real-world \green{sparse} tensors, and based on the results, we provide practitioner's guides to algorithm selection and hyperparameter tuning.
\end{itemize}

\textbf{Reproducibility}: The code and datasets used in the paper are available at \url{https://github.com/DMLab-Tensor/SliceNStitch}.

\textbf{Remarks:} \blue{CPD may not be the best decomposition model for tensors with time modes, and there exist a number of alternatives, such as Tucker, INDSCAL, and DEDICOM (see \cite{kolda2009tensor}).
	%Many tensor decomposition models have been widely studied, such as CP, TUCKER \cite{tucker1966some}, INDSCAL, and DEDICOM. 
	Nevertheless, as a first step, we focus on making CPD `continuous' due to its prevalence and simplicity.
	We leave extending our approach to more models as future work.
}
	%propose an continuous decomposition algorithm for CP, which is one of the most widely used due to its simplicity and effectiveness. Extensions to other decomposition models are left as future work.}

% \kijung{Need to update the organization.}
%We organize the rest of the paper as follows. 
In Section~\ref{sec:prelim}, we introduce some notations and preliminaries.
In Section~\ref{sec:problem}, we provide a formal problem definition.
In Sections \ref{sec:model} and \ref{sec:algo}, we present the model and optimization algorithms of \method, respectively.
In Section \ref{sec:exp}, we review our experiments.
After discussing some related works in Section \ref{sec:related}, we conclude in Section \ref{sec:conclusion}.

\section{Preliminaries}
\label{sec:prelim}
In this section, we introduce some notations and preliminaries.
Some frequently-used symbols are listed in Table~\ref{Tab:Tos}.

\begin{table}
	%	\small
	\centering
	\caption{Table of frequently-used symbols} \label{Tab:Tos}
	\scalebox{1.0}{
	\begin{tabular}{c|l}
		\toprule
		\textbf{Symbol} & \textbf{Definition} \\
%		\midrule
%		\multicolumn{2}{l}{\bf     \\
		\midrule
		$\mat{A}$ & a matrix \\
		
		$\mrow{A}{i}, \mcol{A}{i}$ & $i$-th row of $\mat{A}$, $i$-th column of $\mat{A}$\\
		
		$\mt{\mat{A}}$, $\mi{\mat{A}}$ & transpose of $\mat{A}$, pseudoinverse of $\mat{A}$ \\ 		
		
		$\khatri$, $\hada$ & Khatri-Rao product, Hadamard product \\	
		% 		\iffalse
		% 		$\odot_{i=1}^{\order} \fm{\mat{A}}{i}$ & 
		% 		$\fm{\mat{A}}{1} \odot \cdots \odot \fm{\mat{A}}{i}
		% 		\odot \cdots \odot \fm{\mat{A}}{\order}$ \\
		
		% 		$\ast_{i = 1}^{\order} \ata{\mat{A}}{i}$ & $\ata{\mat{A}}{1} 
		% 		\hada \cdots \hada \ata{\mat{A}}{i}$ \\	
		% 		& $\hada \cdots \hada \ata{\mat{A}}{\order}$\\
		% 		\fi
		\midrule	
		
		$\tensorInF{X}$ & a tensor \\
		
		$\order$ & order of $\tensorInF{X}$ \\
		
		$\modeSize{m}$ & number of indices in the $m$-th mode of $\tensorInF{X}$\\
		
		$\entryi$ & $\indexi$-th entry of $\tensorInF{X}$ \\		
		
		$|\tensorInF{X}|$ & number of non-zeros of $\tensorInF{X}$ \\
		$\lfrosmall{\tensorInF{X}}$ & Frobenius norm of $\tensorInF{X}$ \\
		$\matri{\mat{X}}{m}$ & mode-$m$ matricization of $\tensorInF{X}$\\		
		
		\midrule		
		$\rank$ & rank of CPD \\
		$\fm{\mat{A}}{m}$ & factor matrix of the $m$-th mode \\
		$\reconst{\tensorInF{X}}$ & an approximation of $\tensorInF{X}$ by CPD\\
		
		\midrule
		$\tensorWindow{D}{t}$ & tensor window at time $t$ \\
		$\Delta \tensorInF{X}$ & a change in $\tensorInF{X}$\\
		$\tLen$ & number of indices in the time mode\\
		$\degree{m, i_m}$ & number of non-zeros with $m$-th mode index $i_m$ \\
		$a^{(m)}_{ij}$ & $(i, j)$-th entry of $\fmInF{A}{m}$ \\
	%	$\tensorInF{\bar{X}}$ & a sample of $\tensorInF{X}$ \\
		\bottomrule
	\end{tabular}
	}
\end{table}

\smallsection{Basic Notations}
%\label{sec:prelim:basic}
Consider a matrix $\mat{A}$.
We denote its $i$-th row by $\mrow{\mat{A}}{i}$ and its $i$-th column by $\mcol{\mat{A}}{i}$.
%of a matrix are array acquired by fixing one index and set no restriction on the other. `$:$' indicates all elements of a mode. 
We denote the transpose of $\mat{A}$ by $\mt{\mat{A}}$ and the Moore-Penrose pseudoinverse of $\mat{A}$ by $\mi{\mat{A}}$. 
We denote the Khatri-Rao and Hadamard products by $\khatri$ and $\hada$, respectively. 
See Section I of \cite{appendix} for the definitions of the Moore-Penrose pseudoinverse and both products.

%\smallsection{Notations for Tensors} 
Consider an $\order$-mode \blue{sparse} tensor $\tensorInF{X} \in \threeDims{\modeSize{1}}{\modeSize{2}}{\modeSize{\order}}$, where $\modeSize{m}$ denotes the length of the $m$-th mode. 
We denote each $\indexi$-th entry of $\tensorInF{X}$ by $\entryi$.
We let $|\tensorInF{X}|$ be the number of non-zero entries in $\tensorInF{X}$, and we let $\ltwosmall{\tensorInF{X}}$ be the Frobenius norm  of $\tensorInF{X}$.
We let $\matri{\mat{X}}{m}$ be the matrix obtained by matricizing $\tensorInF{X}$ along the $m$-th mode. 
See Section I of \cite{appendix} for the definitions of the Frobenius norm and matricization.

%Other notations are defined in the following \ref{sec3.B} and \ref{sec4}

%\label{sec:prelim:cp}
\smallsection{CANDECOMP/PARAFAC Decomposition (CPD)} Given an $\order$-mode tensor $\tensorInF{X}\in \threeDims{\modeSize{1}}{\modeSize{2}}{\modeSize{\order}}$ and rank $R\in \mathbb{N}$, 
CANDECOMP/PARAFAC Decomposition (CPD) \cite{hitchcock1927expression} gives a rank-$R$ approximation of $\tensorInF{X}$, expressed as the sum of $R$ rank-$1$ tensors (i.e., outer products of vectors) as follows:
\begin{align}
	\tensorInF{X} & \approx \sumP{r}{R} \bm{a}^{\left(1\right)}_{r} \circ \bm{a}^{\left(2\right)}_{r} \circ \cdots \circ \bm{a}^{\left(M\right)}_{r}, \nonumber \\
	& \equiv \sumP{r}{\rank} \fm{\mat{A}}{1}(:, r) \circ \fm{\mat{A}}{2}(:, r) \circ \cdots \circ \fm{\mat{A}}{\order}(:, r), \nonumber \\
	& \equiv \cpd{\fm{\mat{A}}{1}, \fm{\mat{A}}{2}, \cdots , \fm{\mat{A}}{\order}} \equiv \reconst{\tensorInF{X}}, \label{eq:reconstruct}
\end{align}
where $\bm{a}^{\left(m\right)}_{r} \in \mathbb{R}^{N_m}$ for all $r\in\{1,2,\cdots,R\}$, and $\circ$ denotes the outer product (see Section I of \cite{appendix} for the definition). 
Each $\fm{\mat{A}}{m}\equiv [\bm{a}^{\left(m\right)}_{1} ~ \bm{a}^{\left(m\right)}_{2} ~ \cdots \bm{a}^{\left(m\right)}_{R}] \in \mathbb{R}^{N_m \times R}$ is called the \textit{factor matrix} of the $m$-th mode.

 %$\fm{\mat{A}}{1}, \fm{\mat{A}}{2}, \cdots, \fm{\mat{A}}{M}$, where
%Using factor matrices, we can rewrite the above approximation as follows.
% \begin{comment}
% Given a rank $R$, and factor matrices $\fm{\mat{A}}{i} \in \mathbb{R}^{N_i \times R}$, $i \in [1, N]$, CP decomposition outputs the approximation $\reconst{\tensorInF{X}}$ of $\tensorInF{X}$.
% A rank one tensor is the outer product of $\order$ vectors. 
% CP decompostion is the sum of $\rank$ rank-1 tensors where vectors used in rank-1 tensors are columns in the factor matrices. 
% \end{comment}

%\begin{align*}
%	\tensorInF{X} &\approx \sumP{r}{\rank} \fm{\mat{A}}{1}(:, r) \circ \fm{\mat{A}}{2}(:, r) \circ \cdots \circ \fm{\mat{A}}{\order}(:, r) \\
%	&\equiv \cpd{\fm{\mat{A}}{1}, \fm{\mat{A}}{2}, \cdots , \fm{\mat{A}}{\order}} = \reconst{\tensorInF{X}}
%\end{align*}
%Here $\circ$ denotes the outer product.

CP decomposition aims to find factor matrices that minimize the difference between the input tensor $\tensorInF{X}$ and its approximation $\reconst{\tensorInF{X}}$. That is, it aims to solve Eq.~\eqref{eqn:obj}.
%Each factor matrix can be obtained by solving the following optimization problem:
\begin{equation} 
\min\nolimits_{\fmInF{A}{1}, \cdots, \fmInF{A}{\order}} \lfro{\tensorInF{X} - \cpd{\fmInF{A}{1}, \fmInF{A}{2}, \cdots , \fmInF{A}{\order}}},\label{eqn:obj}
\end{equation}
where $\lfrof$ is the Frobenius norm (see Section I of \cite{appendix} for the definition).

\smallsection{Alternating Least Squares (ALS)}
%\subsection{Alternating Least Squares (ALS)} 
%\label{sec:prelim:als}
Alternating least squares (ALS) \cite{carroll1970analysis,harshman1970foundations} is a standard algorithm for computing \cpdWord of a static tensor.
For each $n$-th mode, $\matri{\cpd{\fmInF{A}{1}, \fmInF{A}{2}, \cdots , \fmInF{A}{\order}}}{n} 
= \fmInF{A}{n}\mt{(\multiKha{A}{n}{\order})}$, and thus
the mode-$n$ matricization of Eq.~\eqref{eqn:obj} becomes
%, i.e., solving Eq.~\eqref{eqn:obj}, which is non-convex.
%Eq.~\eqref{eqn:obj} is non-convex . 
%Consider a matricized version of Eq.~\eqref{eqn:obj}.
%Matricizing both the tensor $\tensorInF{X}$ and its \cpdWord\ gives a hint for the problem.
%Matricized \cpdWord\ $\matri{\cpd{\fmInF{A}{1}, \cdots, \fmInF{A}{\order}}}{n}$ by mode $n \in \{1,\cdots, \order\}$, which approximates the matricized tensor $\matri{\mat{X}}{n}$ by the same mode, can be expressed as follows:
%\begin{equation*}
%\label{eqn:approx_matricized}
%	\matri{\mat{X}}{n} \approx .
%	= \matri{\reconst{\mat{X}}}{n} \nonumber \\
%	&	
%\end{equation*}
%When formulating the problem (\ref{eqn:obj}) with $mode-n$ matricized version of tensor and \cpdWord, it can be seen as follows:
\begin{equation}
\min\nolimits_{\fmInF{A}{1}, \cdots, \fmInF{A}{\order}} \lfro{\matri{\mat{X}}{n} - \fmInF{A}{n} \mt{(\multiKha{A}{n}{\order})}} \label{eqn:obj_matricized}.
\end{equation}
While the objective function in Eq.~\eqref{eqn:obj_matricized} is non-convex, solving Eq.~\eqref{eqn:obj_matricized} only for $\fm{\mat{A}}{n}$ while fixing all other factor matrices is a least-square problem. 
Finding $\fm{\mat{A}}{n}$ that makes the derivative of the objective function in Eq.~\eqref{eqn:obj_matricized} with respect to $\fm{\mat{A}}{n}$ zero leads to the following update rule for  $\fm{\mat{A}}{n}$:
\begin{equation}
\label{eqn:sol_analytic}
	\fm{\mat{A}}{n} \leftarrow \matri{\mat{X}}{n} (\multiKha{A}{n}{\order}) \mi{\{\multiHada{\mat{A}}{n}{\order}\}}.
\end{equation}

ALS performs \cpdWord by alternatively updating each factor matrix  $\fm{\mat{A}}{n}$ using Eq.~\eqref{eqn:sol_analytic} until convergence.
%Thus CP decomposition is earned by alternatively updating the factor matrix of one mode and this process is named as ALS algorithm. 

\section{Problem Definition}
\label{sec:problem}
In this section, we first define multi-aspect data streams.
Then, we describe how it is typically modeled as a tensor and discuss the limitation of the common tensor modeling method.
Lastly, we introduce the problem considered in this work.

\smallsection{Multi-aspect Data Stream and Examples}
We define a multi-aspect data stream, which we aim to analyze, as follows:
\begin{definition}[Multi-aspect Data Stream] \label{defn:stream}
    A {\it multi-aspect data stream} is defined as a sequence of timestamped $M$-tuples $\left\{\pairn\right\}_{n\in\mathbb{N}}$, where
    $i_{1},\cdots,i_{M-1}$ are categorical variables, $v_n\in \mathbb{R}$ is a numerical variable, and $t_n \in \mathbb{N}$ is the time (e.g., Unix timestamp)
    when the $n$-th tuple $e_n$ occurs.
    We assume that the sequence is chronological, i.e., 
    $$t_n \leq t_m ~ \text{if} ~ n<m.$$
    For simplicity, we also assume that the categorical variables are also (assigned to) integers, i.e., 
    $$i_{m}\in\{1,\cdots,\modeSize{m}\}, ~ \forall m\in \{1,\cdots,\order-1\}.$$
\end{definition} 

%\kijung{Give some real-world examples of multi-aspect data. Currently, it can be hard to imagine examples.}
Time-evolving data from various domains are naturally represented as a multi-aspect data stream as follows:
%Time-evolving multi-aspect data from various domains are naturally represented as a multi-aspect data stream as follows:
\begin{itemize}[leftmargin=*]
    \item {\bf Traffic History}: 
    Each $3$-tuple $e_{n}$ = (source, destination, $1$) indicates a trip that started at time $t_n$. % is the time when the trip started. % at the source.
    \item {\bf Crime History}:
    Each $3$-tuple $e_{n}$ = (location, type, $1$) indicates an incidence of crime at time $t_n$ %. is the time of the incidence.
    \item {\bf Purchase History}: 
    Each $4$-tuple $e_{n}$ = (user, product, color, quantity) indicates a purchase at time $t_n$.% is the time of the purchase.
\end{itemize}

\smallsection{Common Tensor Modeling Method and its Limitations}
Multi-aspect data streams are commonly modeled as tensors to benefit from powerful tensor analysis tools (e.g., CP decomposition)
\cite{sun2006window,sun2008incremental,xu2019anomaly,nion2009adaptive,zhang2018variational,zhou2016accelerating,zhou2018online,gujral2018sambaten}.
Specifically, a multi-aspect data stream is modeled as an $\order$-mode tensor $\tensorInF{X} \in \threeDimsBack{\modeSize{1}}{\modeSize{\order-1}}{W}$, where $W$ is the number of indices in the time mode (i.e., the $\order$-th mode).
For each $w\in\{1,\cdots,W\}$, let $\tensorInF{G}_w\in  \threeDims{\modeSize{1}}{\modeSize{2}}{\modeSize{\order-1}}$ be the $(\order-1)$-mode tensor obtained from $\tensorInF{X}$  by fixing the $\order$-th mode index to $w$.
That is, $\tensorInF{X} \equiv \tensorInF{G}_{1} ~\concat~ \cdots ~\concat~ \tensorInF{G}_{W-1} ~\concat~ \tensorInF{G}_{W},$
where $\concat$ denotes concatenation.
Each tensor $\tensorInF{G}_{w}$ is the sum over $T$ time units (i.e., $T$ is the period) ending at $wT$.
That is, each $(\indexj)$-th entry of $\tensorInF{G}_{w}$ is the sum of $v_n$ over all $\order$-tuples $e_n$ where $i_m=j_m$ for all $m\in\{1,\cdots,\order-1\}$ and $t_n\in (wT-T,wT]$.
See Figs.~\ref{subfig:coarse} and \ref{subfig:fine} for examples where $T$ is an hour and a second, respectively.
As new tuples in the multi-aspect data stream arrive, a new $(\order-1)$-mode tensor is added to $\tensorInF{X}$ once per period $T$.

Additionally, in many previous studies on window-based tensor analysis \cite{sun2006window,sun2008incremental,xu2019anomaly,zhang2018variational}, the oldest $(\order-1)$-mode tensor is removed from $\tensorInF{X}$ once per period $T$ to fix the number of indices in the time mode to $W$.
That is, 
at time $t=W'T$ where $W'\in \{W+1, W+2, \cdots\}$, $\tensorInF{X} \equiv \tensorInF{G}_{W'-W+1} ~\concat~ \cdots ~\concat~ \tensorInF{G}_{W'-1} ~\concat~ \tensorInF{G}_{W'}$.

A serious limitation of this widely-used tensor model is that the tensor $\tensorInF{X}$ changes only once per period $T$ while the input multi-aspect data stream grows continuously.
Thus, it is impossible to analyze multi-aspect data streams continuously in real time in response to the arrival of each new tuple.

\smallsection{Problem Definition}
How can we continuously analyze multi-aspect data streams using tensor-analysis tools, specifically, \cpdWord?
We aim to answer this question, as stated in Problem~\ref{defn:problem:informal}. %, which is also given in \change{}
\begin{problem}[Continuous CP Decomposition] \label{defn:problem:informal}
	\textbf{(1) Given:} a multi-aspect data stream, \textbf{(2) to update:} its CP decomposition instantly in response to each new tuple in the stream, \textbf{(3) without having to wait} for the current period to end. 
\end{problem}
Note that, as discussed in Section~\ref{sec:intro} and shown in Fig.~\ref{fig:motivation}, an extremely short period $T$ cannot be a proper solution since it extremely decreases the \blue{fitness} of \cpdWord and at the same time increases the number of parameters.

%A multi-aspect data stream is naturally represented as an $M$-mode tensor.
%We build a tensor with the recent part of the multi-aspect data stream. 

\begin{comment}
Specifically, a tuple in a multi-aspect data stream contributes to an entry of the tensor.
Consider a tuple $(e_n = (i_1, \cdots, i_{\order - 1}, v), t_n)$, $i_m$ corresponds to the $m$-th mode index for all $m \in \{1, \cdots, \order - 1\}$. 
The second value of the tuple, $t_n$, is used to find the last mode index of the tensor.
Let the time dimension of the tensor as $W$, the point of the $w$-th time mode index as $s_w$($w \in \{1, \cdots, W\}$), and the size of the area covered by one time mode index as $B$.
We make the tensor end at the current time $t_{curr}$ so that $s_W = t_{curr}$.
The time index for the tuple is $w$ if and only if $s_w - B < t_n \leq s_w$.
In short, the tensor value at an index $(i_1, \cdots, i_{\order}, j)$ is the sum of $v$ in tuples which correspond to the index.
\end{comment}

%Note that repeatedly rebuilding the input tensor at every time point $t$ and performing CP decomposition from scratch is computationally prohibitive. 
%In the following section, we propose a novel data model that incrementally updates the previous tensor to make it the latest tensor which ends at the current time $t_{curr}$. 
%In this way, we get the input tensor of Problem 1 in a time-efficient manner.

\section{Proposed Data Model and Implementation}
\label{sec:model}
In this section, we propose the continuous tensor model and its efficient implementation.
This data model is one component of \method. See Section~\ref{sec:algo} for the other component.
\subsection{Proposed Data Model: Continuous Tensor Model}
\vspace{-0.5mm}
\label{sec:model:defn}

We first define several terms which our model is based on.
\begin{definition}[Tensor Slice] \label{dfn:slice}
    Given a multi-aspect data stream (Definition~\ref{defn:stream}), for each timestamped $\order$-tuple $\pairn$, we define the \textit{tensor slice} 
	$\tensorInF{Z}_n \in \twoDims{\modeSize{1}}{\modeSize{\order-1}}$ as an $\left(\order-1\right)$-mode tensor where the $\indeximinus$-th entry is $v_{n}$ and the other entries are zero.
%	It can be also considered as a $\left(M+1\right)$-mode tensor by concatenating it in the $\left(M+1\right)$-th dimension.
\end{definition} 

%A {\it multi-aspect data stream} is defined as a sequence of $\order$-mode tensors $\left\{\tensorInF{X}_t\right\}_{t\in\mathbb{N}}$ where $\tensorInF{X}_t \in \twoDims{\modeSize{1}}{\modeSize{\order}}$ represents data arrived at time $t$.
%	It can be also considered as a $\left(M+1\right)$-mode tensor by concatenating it in the $\left(M+1\right)$-th dimension.

%\kijung{Please change the definitions below accordingly}

\begin{definition}[Tensor Unit] \label{defn:unit}
	Given a multi-aspect data stream, a time $t$, and the period $\batch$, we define the \textit{tensor unit} $\tensorInF{Y}_t$ as 
	\begin{align*}
		\tensorInF{Y}_t \equiv \sum\nolimits_{t_n \in \left(t- \batch, t\right]} \tensorInF{Z}_n.
	\end{align*}
\end{definition}

\noindent That is, $\tensorInF{Y}_t \in \twoDims{\modeSize{1}}{\modeSize{\order-1}}$ is an aggregation of the tuples occurred within the half-open interval $\left(t-\batch, t\right]$.
	
\begin{definition}[Tensor Window] \label{defn:tensor_window}
    Given a multi-aspect data stream, a time $t$, the period $\batch$, and the number of time-mode indices $W$,
    we define the \textit{tensor window} $\tensorWindow{D}{t}$ as 
	\begin{align*}
		\tensorWindow{D}{t} \equiv \tensorInF{Y}_{t-\left(\tLen-1\right)\batch} ~\concat~ \cdots ~\concat~ \tensorInF{Y}_{t-\batch} ~\concat~ \tensorInF{Y}_t,
	\end{align*}
	where $\concat$ denotes concatenation.
%	which can be also viewed as a concatenated tensor in the $\order$-th dimension.
\end{definition}
\noindent 	That is, $\tensorWindow{D}{t} \in \threeDimsBack{\modeSize{1}}{\modeSize{\order-1}}{\tLen}$ concatenates the $\tLen$ latest tensor units. 

\begin{figure}[t]
	\centering
	\vspace{-2mm}
	\includegraphics[width=\linewidth]{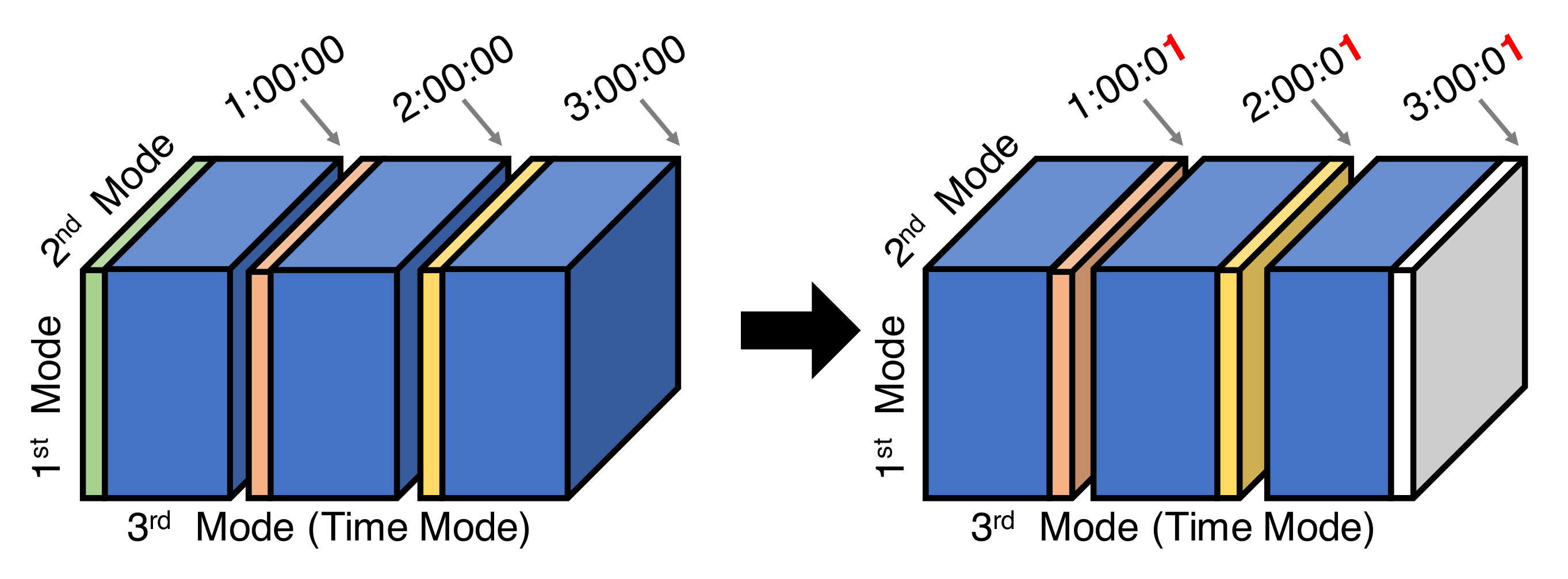}
	\caption{\label{fig:changes_tensor_window}
		{\bf Example of the continuous tensor model.}
		The starting points of tensor units (whose length is an hour) change adaptively based on the current time (which is 3:00:01).
	}
%		The starting point of each period (whose length is an hour in this example, changes `continuously'  and adaptively (every second in this example).  }
	%	Changes in the tensor window of length 3. Tensor slices at the front boundary are relayed to the rear boundary of the front tensor unit.}
\end{figure}

%Assume that tensors in the multi-aspect data stream are very sparse.
%In that case, \cpdWord of the data stream may not work due to the reason that we explained in Section \ref{sec:intro}.
%Therefore, tensors must be merged to form batches, which are less sparse than the original tensor, and we call each batch as a tensor unit.

The main idea of the continuous tensor model is to adaptively adjust the starting point (or equally the end point) of each tensor unit based on the current time, as described in Definition~\ref{defn:model} and Fig.~\ref{fig:changes_tensor_window}.
\begin{definition}[Continuous Tensor Model] \label{defn:model}
    In the {\it continuous tensor model}, given a multi-aspect data stream, the period $\batch$, and the number of time-mode indices $W$,
    the modeled tensor evolves from $\tensorWindow{D}{t-dt}$ to $\tensorWindow{D}{t}$ at each time $t$, where $dt$ represents the infinitesimal change in time.
    %the modeled tensor evolves from $\tensorWindow{D}{t-1}$ to $\tensorWindow{D}{t}$ at each time $t$.
\end{definition}

Note that in the continuous tensor model, the modeled tensor changes `continuously', i.e., once per minimum time unit in the input multi-aspect data stream (e.g., a millisecond), while the modeled tensor changes once per period $\batch$ in the typical models, discussed in Section~\ref{sec:problem}.

\subsection{Event-driven Implementation of Continuous Tensor Model}
\vspace{-0.5mm}
\label{sec:model:impl}

In the continuous tensor model, it is crucial to efficiently update the modeled tensor, or in other words, efficiently compute the change in $\tensorWindow{D}{t}$ in the modeled tensor.
This is because repeatedly rebuilding the modeled tensor $\tensorWindow{D}{t}$ at every time $t$ from scratch is computationally prohibitive.

We propose an efficient event-driven implementation of the continuous tensor model.
Let $\tensorInF{X}=\tensorWindow{D}{t}$ for simplicity.
Our implementation, described in Algorithm~\ref{alg:model}, is based on the observation that each tuple $\pairn$ in the input multi-aspect data stream causes the following events:
\begin{enumerate} [label=\textbf{S.\arabic*}]
    \item \label{enum:event:first} At time $t=t_n$, the value $v_n$ is added to $\entryiW$.
    \item \label{enum:event:second} At time $t=t_n+w\batch$ for each $w\in\{1,\cdots,W-1\}$, the value $v_n$ is subtracted from $\entryiWwo$ and then added to $\entryiWw$.
    \item \label{enum:event:third} At time $t=t_n+\tLen\batch$, the value $v_n$ is subtracted from  $\entryione$.
\end{enumerate}
As formalized in Theorem~\ref{thm:model:complexity:time}, our implementation maintains the modeled tensor $\tensorInF{X}=\tensorWindow{D}{t}$ up-to-date at each time $t$ by performing $O(\order \tLen)$ operations per tuple in the input stream.
Note that $\order$ and $\tLen$ are usually small numbers, and if we regard them as constants, the time complexity becomes $O(1)$, i.e. processing each tuple takes constant time.

\begin{thm}[Time Complexity of the Continuous Tensor Model] \label{thm:model:complexity:time}
In Algorithm~\ref{alg:model}, the time complexity of processing each timestamped $\order$-tuple is $O(\order \tLen)$. 
\end{thm}
\begin{proof}
For each timestamped $\order$-tuple, $\tLen + 1$ events occur. Processing an event (lines \ref{alg:model:event1:start}-\ref{alg:model:event1:end} or \ref{alg:model:event2:start}-\ref{alg:model:event2:end}) takes $O(\order)$ time.
\end{proof}

\begin{thm}[Space Complexity of the Continuous Tensor Model] \label{thm:model:complexity:space}
In Algorithm~\ref{alg:model}, the space complexity is
$$O\left(\order\cdot \max_{t\in \mathbb{R}} |\{n\in \mathbb{N} : t_n\in (t-\tLen\batch,t] \}|\right).$$
\end{thm}
\begin{proof}
We call $S_{t}\equiv\{n\in \mathbb{N} : t_n\in (t-\tLen\batch,t] \}$ the set of active tuples at time $t$.
The number of non-zeros in $\tensorInF{X}=\tensorWindow{D}{t}$ is upper bounded by $|S_{t}|$, and thus the space required for $\tensorInF{X}$ is $O(\order\cdot |S_{t}|)$. Since at most one event is scheduled for each active tuple, the space required for storing all scheduled events is also $O(\order\cdot |S_{t}|)$. 
\vspace{-1mm}
\end{proof}

\begin{algorithm}[t] 
	\small
	\caption{\label{alg:model} Event-driven Implementation of the Continuous Tensor Model}
	\KwIn{(1) a multi-aspect data stream, (2) period $\batch$, \\
		\qquad\quad (3) number of indices in the time mode $\tLen$
	}
	\KwOut{up-to-date tensor window $\tensorInF{X}=\tensorWindow{D}{t}$}
	
	initialize $\tensorInF{X}$ to a zero tensor $\in \threeDimsBack{\modeSize{1}}{\modeSize{\order-1}}{\tLen}$ \;
	wait until an event $E$ occurs \label{alg:model:wait} \; 
	\If{$E=$ arrival of $\pairn$}
	{
		add $v_n$ to $\entryiW$ \; \label{alg:model:event1:start}
		schedule the $1$-st update for $\pairnshort$ at time $t_n+T$ \label{alg:model:event1:end} \;
	}
	\If{$E=$ $w$-th update for $\pairn$}
	{
		subtract $v_n$ from $\entryiWwo$ \; \label{alg:model:event2:start}
		\If{$w<\tLen$}{
			add $v_n$ to $\entryiWw$ \;
			schedule the $(w+1)$-th update for $\pairnshort$ at time $t_n+(w+1)T$ \label{alg:model:event2:end} \;
		}
	}
	\textbf{goto} Line~\ref{alg:model:wait}
\end{algorithm}

\section{Proposed Optimization Algorithms}
\label{sec:algo}
In this section, we present the other part of \method.
We propose a family of online optimization algorithms for \cpdWord of \blue{sparse tensors} in the continuous tensor model.
As stated in Problem~\ref{defn:problem:cpd}, they aim to update factor matrices fast and accurately in response to each change in the tensor window.  
%event that changes the current tensor window $\tensorInF{X}=\tensorWindow{D}{t}$. % (i.e., $\tensorWindow{D}{t}$).
%in the tensor window with fast speed and comparable error so that we can get it at any time.
\begin{problem}[Online CP Decomposition of \blue{Sparse Tensors} in the Continuous Tensor Model] \label{defn:problem:cpd} \textbf{(1) Given:} 
	\begin{itemize}
	    \item the current tensor window $\tensorInF{X}=\tensorWindow{D}{t}$, 
	    \item factor matrices $\fm{\mat{A}}{1}, \fm{\mat{A}}{2}, \cdots , \fm{\mat{A}}{\order}$ for $\tensorInF{X}$,
	    \item an event for $\pairn$ occurring at $t$,
	\end{itemize}
\textbf{(2) Update:} the factor matrices in response to the event, \\
\textbf{(3) To Solve:} the minimization problem in Eq.~\eqref{eqn:obj}.
\end{problem}

Below, we use $\Delta \tensorInF{X}$ to denote the change in $\tensorInF{X}$ due to the given event. By \ref{enum:event:first}-\ref{enum:event:third} of Section \ref{sec:model:impl}, Definition~\ref{defn:delta} follows.
\begin{definition}[Input Change] \label{defn:delta}
	The input change $\Delta \tensorInF{X}\in \threeDimsBack{\modeSize{1}}{\modeSize{\order-1}}{\tLen}$ is defined as the change in $\tensorInF{X}$ due to an event for $\pairn$ occurring at $t$, i.e.,
	
	\begin{itemize}[leftmargin=*]
		\item {[If $t=t_n$]} The $\indexiW$-th entry of $\Delta \tensorInF{X}$ is $v_n$, and the other entries are zero.
		\item {[If $t=t_n + wT$ for  $1\leq w<W$]} The $(\indexiParam{W - w})$-th and $(\indexiParam{W - w + 1})$-th entries of $\Delta \tensorInF{X}$ are $v_n$ and $-v_n$, respectively, and the others are zero.
		\item {[If $t=t_n + WT$]} The $(\indexiParam{1})$-th entry of $\Delta \tensorInF{X}$ is $-v_n$, and the others are zero.
	\end{itemize}

\end{definition}

%By \ref{enum:event:second}, 
%By \ref{enum:event:third}, 

%\kijung{From here, I have to proofread ... }
%\kijung{Let's decide our ultimate method after seeing the experimental results}
We first introduce \walslong (\wals), which 
naively applies ALS to Problem~\ref{defn:problem:cpd}.
Then, we present \selalslong (\selals) and \hyalslong (\hyals).
Lastly, we propose our main methods, \selccd and \hyccd. 
%, which are stable versions of \selals and \hyals.

%: \textit{\mainMethod} in section 4.2, and its variant: \textit{\fastMethod} which further improve the time complexity in section 4.3.
\subsection{\walslong (\wals)} \label{sec:algo:wals}
\vspace{-0.5mm}
When we apply ALS to Problem~\ref{defn:problem:cpd}, the factor matrices for the current window $\tensorInF{X}$ are strong initial points. %for obtaining the updated factor matrices, and 
Thus, a single iteration of ALS is enough to achieve high \blue{fitness}.
The detailed procedure of \wals is given in Algorithm~\ref{alg:wals}.
In line~\ref{line:wals:normalization}, we normalize\footnotemark\ the columns of each updated factor matrices for balancing the scales of the factor matrices. % without increasing the time complexity in the case of \wals.

\footnotetext{
	Let the $r$-th entry of $\lambda \in \mathbb{R}^R$ as $\lambda_r$ and the $r$-th column of $\fmInF{A}{m}$ be $\bm{a}^{(m)}_r$. 
	We set $\lambda_r$ to $\ltwosmall{\bm{a}^{(m)}_r}$ and set $\bm{{\bar{a}}}^{(m)}_r$ to $\bm{a}^{(m)}_r / \lambda_r$ for $r = 1, \cdots, R$.
	Then, $\tensorInF{X}$ is approximated as $\sumP{r}{R} \lambda_r \bm{{\bar{a}}}^{\left(1\right)}_{r} \circ \bm{{\bar{a}}}^{\left(2\right)}_{r} \circ \cdots \circ \bm{{\bar{a}}}^{\left(M\right)}_{r}$.
}

\begin{algorithm}[t]
    \small
	\caption{\label{alg:wals} \wals: Naive Extension of ALS.}
	\KwIn{(1) current tensor window $\tensorInF{X}$, (2) change $\Delta \tensorInF{X}$, \\
		\qquad\quad (3) column-normalized factor matrices $\matset{\fmInF{\bar{A}}{m}}{\order}$, \\
		\qquad\quad (4) $\matset{\ata{\mat{{\bar{A}}}}{m}}{\order}$
	}
	\KwOut{Updated $\matset{\fmInF{{\bar{A}}}{m}}{\order}$,  $\matset{\ata{\mat{{\bar{A}}}}{m}}{\order}$, and $\lambda$}
		
	%Compute $\matset{\ata{\mat{A}}{m}}{\order}$
	
	\For{$m = 1, \cdots, \order$}
	{
		$\mat{U} \leftarrow \matri{\left(\mat{X} + \Delta \mat{X}\right)}{m} \left( \multiKhaN{{\bar{A}}}{m}{\order} \right)$ \label{line:wals:mttkrp}
		
		$\mat{H} \leftarrow \ast_{n \neq m}^{\order}\ata{\mat{{\bar{A}}}}{n}$  \label{line:wals:hadamard}
		
		$\fmInF{A}{m} \leftarrow \mat{U} \mi{\mat{H}}$ \label{line:wals:initialA}	\tcp{by Eq.~(\ref{eqn:sol_analytic})} 
		
		$\lambda \leftarrow$ $\ltwotext$ norms of the columns of $\fmInF{{A}}{m}$
		\label{line:wals:lambda} \tcp{$\lambda\in\mathbb{R}^R$} 
		
		$\fmInF{{\bar{A}}}{m} \leftarrow$ column normalization of  $\fmInF{{A}}{m}$ \label{line:wals:normalization}

	%	\change{Normalize columns of $\fmInF{A}{m}$ (storing norms as $\lambda$)} \label{line:wals:normalization}\footnotemark[1]
		
		Update $\ata{\mat{{\bar{A}}}}{m}$ \label{line:wals:update}
	}
	
	\Return $\matset{\fmInF{{\bar{A}}}{m}}{\order}$, $\matset{\ata{\mat{{\bar{A}}}}{m}}{\order}$, and $\lambda$
\end{algorithm}

\smallsection{Pros and Cons}
For each event, \wals accesses every entry of the current tensor window and updates every row of the factor matrices. Thus, it suffers from high computational cost, as formalized in Theorem~\ref{thm:time:wals}, while it maintains a high-quality solution (i.e., factor matrices).
% is not appropriate for the streaming algorithm because of its scalability. 
%Considering all entries in the tensor window for optimization and updating all rows in factor matrices lead to the large time complexity.

\begin{thm}[Time complexity of \wals]
Let $\modeSize{\order}=W$. Then,
the time complexity of \wals is \label{thm:time:wals}
\begin{equation}
    O \bigg( \order^2\rank | \tensorInF{X} + \Delta \tensorInF{X}| + M^2R^2 + 
    \order R^3 + \sum_{m=1}^{\order}\modeSize{m} \rank^2 \bigg).
    \label{eq:complexity:wals}
\end{equation}
\end{thm}
\begin{proof}
\begin{comment}
	The operation in line~\ref{line:wals:mttkrp}, called the \textit{matricized tensor times Khatri-Rao product}, takes $O(MR \left| \tensorInF{X} + \Delta \tensorInF{X} \right|)$ time.\footnote{It takes $O(MR)$ time per nonzero in $\tensorInF{X} + \Delta \tensorInF{X}$, as shown in \cite{bader2008efficient}.}
	%, and thus it $
	%Thus, repeating it $M$ times takes as much as $O \left( \order^2\rank \left| \tensorInF{X} + \Delta \tensorInF{X} \right| \right)$.
	Computing $\mat{H}\in\mathbb{R}^{R\times R}$ in line~\ref{line:wals:hadamard} takes $O(MR^2)$ time as each entry of $\mat{H}$ requires $O(M)$ products. 
	%Thus, it takes $O((MR)^2)$ for total $M$ outermost iterations.
	In line~\ref{line:wals:initialA}, computing the pseudoinverse $\mi{\mat{H}}$ takes $O(R^3)$ time \cite{courrieu2008fast},
	and multiplying $\mat{U}\in\mathbb{R}^{N_m \times R}$ and $\mi{\mat{H}}\in\mathbb{R}^{R \times R}$ takes $O(\modeSize{m}R^2)$ time. 
	The time complexities of lines \ref{line:wals:lambda}-\ref{line:wals:update} are dominated by those of lines~\ref{line:wals:mttkrp}-\ref{line:wals:initialA}.
	%for each $m \in \{1, \cdots, \order - 1\}$ and $WR^2$ for $m = \order$.
	%Other lines take similar or less time to run. }
	Hence, the time complexity of Algorithm~\ref{alg:wals}, which repeats lines \ref{line:wals:mttkrp}-\ref{line:wals:update} with $m=\{1,\cdots, M\}$ is Eq.~\eqref{eq:complexity:wals}.	
\end{comment}
See Section II.A of the online appendix \cite{appendix}.
\end{proof}

\begin{algorithm}[t]
	\small
	\caption{\label{alg:total} Common Outline of \selals, \selccd, \hyals, and \hyccd.}
	\KwIn{(1) current tensor window  $\tensorInF{X} = \tensorInF{D}(t, W)$, \\
		\qquad\quad (2) change $\Delta \tensorInF{X}$ due to an event for \\
		\qquad\qquad \ $\pairn$ occurring at $t$, \\
		\qquad\quad (3) factor matrices $\matset{\fmInF{A}{m}}{\order}$ \\
		\qquad\quad (4) $\matset{\ata{\mat{A}}{m}}{\order}$, (5) period $T$
	}	
	\KwOut{updated $\matset{\fmInF{A}{m}}{\order}$ and $\matset{\ata{\mat{A}}{m}}{\order}$}
	
	\SetKwFunction{updateRow}{updateRow}
	
	%\tcp{Update the factor matrix of the temporal mode}

	$\matset{\mt{\fmInF{A}{m}_{prev}}\fmInF{A}{m}}{\order} \leftarrow \matset{\ata{\mat{A}}{m}}{\order}$
	\label{line:total:start} \tcp{used only in \hyals and \hyccd}

	%   $\Delta \tensorInF{X} \leftarrow$ the zero tensor whose dimension is the same as $\tensorInF{X}$
	
	$w \leftarrow (t - t_n) / T$ \ \tcp{time-mode index}

	%    \tcp{Update the row for the subtracted unit}
	\If{$w > 0$ \label{alg:total:time:start}}{
		%	    $\Delta x_{\indexiParam{W - w + 1}} \leftarrow -v_n$
		
		\updateRow{$\order$, $W - w + 1$, $\cdots$} \tcp{Alg.~\ref{alg:vec} or \ref{alg:ccd}}
	} 
	
	%	\tcp{Update the row for the added unit}
	\If{$w < W$}{
		
		%	    $\Delta x_{\indexiParam{W - w}} \leftarrow v_n$
		
		\updateRow{$\order$, $W - w$, $\cdots$}   \tcp{Alg.~\ref{alg:vec} or \ref{alg:ccd}} \label{alg:total:time:end}
	}
	
	\For{$m \leftarrow 1, \cdots, \order - 1$ \label{alg:total:nontime:start}}
	{
		\updateRow{$m$, $i_m$}  \tcp{Alg.~\ref{alg:vec} or \ref{alg:ccd}}
		\label{line:total:end} \label{alg:total:nontime:end}
	}
	%	$\tensorInF{X} \leftarrow \tensorInF{X} + \Delta \tensorInF{X}$
	\Return $\matset{\fmInF{{A}}{m}}{\order}$ and $\matset{\ata{\mat{{A}}}{m}}{\order}$
	
	%\Return $\matset{\fmInF{{\bar{A}}}{m}}{\order}$ and $\matset{\ata{\mat{{\bar{A}}}}{m}}{\order}$	
	
\end{algorithm}
\begin{figure}
	\vspace{-5mm}
	\centering
	\subfloat{\includegraphics[width=\linewidth]{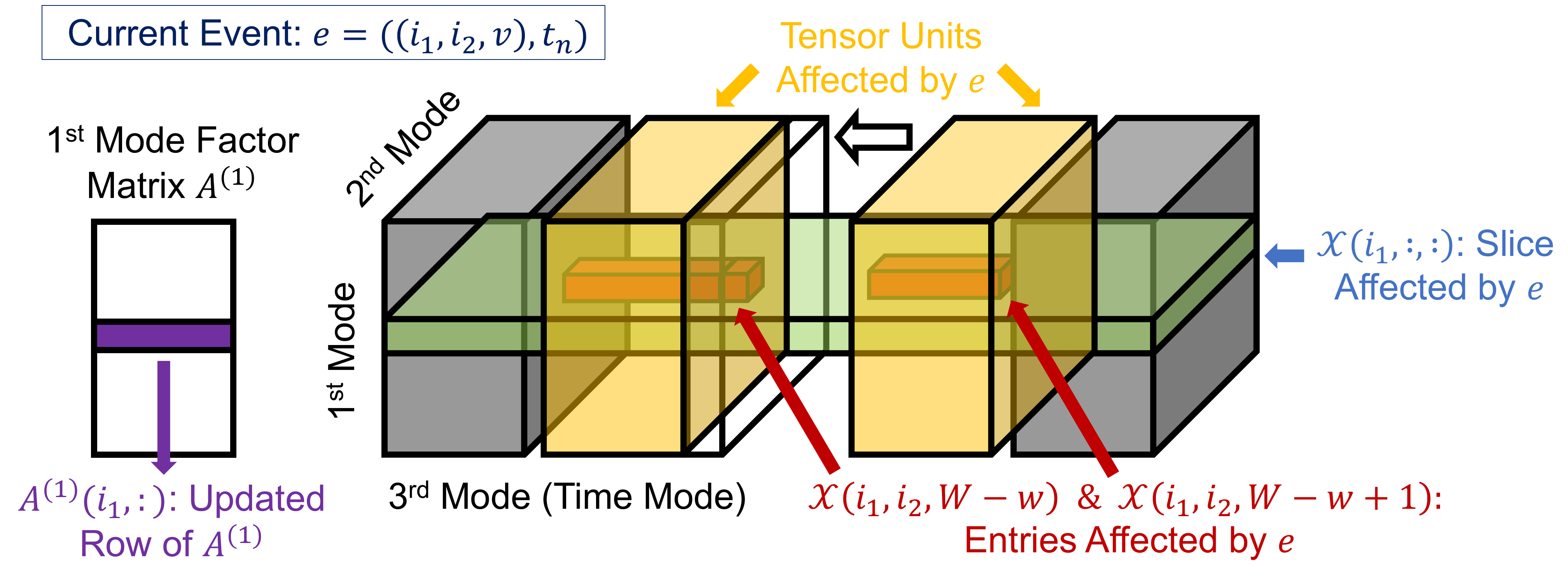}} \\ \vspace{-2mm} 
	\caption{\label{fig:running_example} \blue{\textbf{Example: updating a row of $\fmInF{A}{1}$.} 
			\selals and \hyals update $\fmInF{A}{1}(i_1, :)$ at once.
			\selccd and \hyccd update it entry by entry.
			\hyals and \hyccd sample $\theta$ entries from those in $ \tensorInF{X}(i_1, :, :)$ if more than $\theta$ entries there are non-zero (i.e., if $deg(1,i_1)>\theta$).}}
	\vspace{-1mm}
\end{figure}

\subsection{\selalslong (\selals)} \label{sec:algo:selals}
\vspace{-0.5mm}
We propose \selals, a fast algorithm for Problem~\ref{defn:problem:cpd}.
The outline of \selals is given in Algorithm~\ref{alg:total}, and the update rules are given in Algorithm~\ref{alg:vec} \blue{with a running example in Fig.~\ref{fig:running_example}.}
%, which outlines all the proposed methods except \wals.
The key idea of \selals is to update only the rows of factor matrices that approximate changed entries of the tensor window.
Starting from the maintained factor matrices, \selals  updates such rows of the time-mode factor matrix (lines~\ref{alg:total:time:start}-\ref{alg:total:time:end}) and then such rows of the non-time mode factor matrices (lines~\ref{alg:total:nontime:start}-\ref{alg:total:nontime:end}).
Below, we describe the update rules used. % in \selals.
%In response to a change $\Delta \tensorInF{X}$ of $\tensorInF{X}$, \selals updates the time-mode factor matrix
% gives the outline of our methods except \wals.

\smallsection{Time Mode} %\label{sec:algo:selals:temp} 
Real-world tensors are typically modeled so that the time mode of $\tensorInF{X}$ has fewer indices than the other modes.
Thus, each tensor unit (see Definition~\ref{defn:unit}) in $\tensorInF{X}$ is likely to contain many non-zeros, and thus even updating only few rows of the time-mode factor matrix (i.e., $\fmInF{A}{\order}$) is likely to incur considerable computational cost. To avoid the cost, \selals employs an approximated update rule.

From Eq.~(\ref{eqn:sol_analytic}), the following update rule for the time-mode factor matrix follows:
\begin{equation} \label{eqn:selals:sol_temporal_origin}
\fmInF{A}{\order} \leftarrow \matri{\left( \mat{X} + \Delta \mat{X} \right)}{\order} \fmInF{K}{\order} \mi{\fmInF{H}{\order}},
\end{equation}
where $\fmInF{K}{\order} = \multiKhaEq{A}{1}{\order-1}$ and $\fmInF{H}{\order} = \multiHadaEq{\mat{A}}{1}{\order-1}$.
If we assume that the approximated tensor $\reconst{\tensorInF{X}}$ in Eq.~\eqref{eq:reconstruct} approximates $\tensorInF{X}$ well, then Eq.~(\ref{eqn:selals:sol_temporal_origin}) is approximated to Eq.~\eqref{eqn:selals:sol_temporal_middle}.
%\iffalse
%\begin{align}
%    \begin{split}
%        \fmInF{A}{\order} &\leftarrow \matri{\reconst{\mat{X}}}{\order} \fmInF{K}{\order} \mi{\fmInF{H}{\order}} + \Delta \matri{\mat{X}}{\order} \fmInF{K}{\order} \mi{\fmInF{H}{\order}} \\
%    	 &= \fmInF{A}{\order} \mt{\fmInF{K}{\order}} \fmInF{K}{\order} \mi{\fmInF{H}{\order}} + \Delta \matri{\mat{X}}{\order} \fmInF{K}{\order} \mi{\fmInF{H}{\order}}
%    \end{split}
%\end{align}
%\fi
\begin{equation} \label{eqn:selals:sol_temporal_middle}
\fmInF{A}{\order} \leftarrow \fmInF{A}{\order} \mt{\fmInF{K}{\order}} \fmInF{K}{\order} \mi{\fmInF{H}{\order}} + \Delta \matri{\mat{X}}{\order} \fmInF{K}{\order} \mi{\fmInF{H}{\order}}.
\end{equation}
%Remark that $\matri{\reconst{\mat{X}}}{\order}$ is the same as $\fmInF{A}{\order}\mt{\fmInF{K}{\order}}$.
By a property of the Khatri-Rao product \cite{kolda2009tensor}, Eq.~\eqref{eqn:selals:KtoH} holds.
\begin{align} \label{eqn:selals:KtoH}
    \begin{split}
        \mt{\fmInF{K}{\order}} \fmInF{K}{\order} &= \mt{\left(\multiKhaEq{A}{1}{\order-1}\right)} \left(\multiKhaEq{A}{1}{\order-1}\right) \\
	    &= \multiHadaEq{\mat{A}}{1}{\order-1} = \fmInF{H}{\order}.
    \end{split}
\end{align}
By Eq.~\eqref{eqn:selals:KtoH}, Eq.~\eqref{eqn:selals:sol_temporal_middle} is reduced to Eq.~\eqref{eqn:selals:sol_temporal}.
\begin{align}
    \label{eqn:selals:sol_temporal}
	\fmInF{A}{\order} \leftarrow \fmInF{A}{\order} + \Delta \matri{\mat{X}}{\order} \fmInF{K}{\order} \mi{\fmInF{H}{\order}}.
\end{align}
Computing Eq.~(\ref{eqn:selals:sol_temporal}) is much cheaper than computing Eq.~(\ref{eqn:selals:sol_temporal_origin}). Since $\Delta \matri{\mat{X}}{\order}$ contains at most two non-zeros (see Problem~\ref{defn:problem:cpd}), computing
$\Delta \matri{\mat{X}}{\order} \fmInF{K}{\order}$ in Eq.~(\ref{eqn:selals:sol_temporal}) takes $O(\order\rank)$ time.
Due to the same reason, at most two rows of $\fmInF{A}{\order}$ are updated.

\smallsection{Non-time Modes} %\label{sec:algo:selals:nonTemp}
When updating $\fmInF{A}{m}$, while fixing the other factor matrices, the objective Eq.~\eqref{eqn:obj_matricized} becomes  Eq.~\eqref{eqn:selals:nontemporal_origin}.
\begin{equation} \label{eqn:selals:nontemporal_origin}
\min_{\fmInF{A}{m}} \lfro{\matri{\left( \mat{X} + \Delta \mat{X} \right)}{m} - \fmInF{A}{m} \mt{\left(\multiKhaN{A}{m}{\order}\right)}}.
\end{equation}
Note that $\Delta \tensorInF{X}$ contains up to two non-zeros, which are the entries changed in $\tensorInF{X}$, and their $m$-th mode indices are $i_m$ (see Problem~\ref{defn:problem:cpd}).
%, the only difference between $\matri{\mat{X}}{m}$ and $\matri{\left( \mat{X} + \Delta \mat{X} \right)}{m}$ is the $i_m$-th row.
By Eq.~(\ref{eqn:obj_matricized}), only the $i_m$-th row of $\fmInF{A}{m}$ is used to approximate the changed entries, and thus \selals updates only the row.
%Thus the $i$-th row of $\fmInF{A}{m}$ should be updated by priority.

\begin{comment}
    Here only the $\Delta \tensorInF{X}=\tensorInF{I}(\indexiw,v)$ entries of $\matri{D(t, W)}{n}$ whose indices are same with that of non zero entries in $\Delta \matri{D(t, W)}{n}$ are changed from $\matri{D(t-1, W)}{n}$ by (5). Thus rows of $\fmInF{A}{n}$ contribute to fit those entries should be updated by priority. We define the indices of these rows as selected indices.
    
    \begin{definition} {(SELECTED INDICES).}
    	A set of selected indices $S_{t}$ for variance tensor $\Delta \tensorWindow{D}{t}$ and mode $n$ is composed with integers which correspond to mode-n index of non-zero entries in $\Delta \tensorWindow{D}{t}$.
    \end{definition}
\end{comment}

If we fix all the other variables except $\fmInF{A}{m}(i_m, :)$, the problem in Eq.~\eqref{eqn:selals:nontemporal_origin} becomes the problem in Eq.~\eqref{eqn:selals:nontemproal_final}.
\begin{multline} \label{eqn:selals:nontemproal_final}
\small
    \min_{\fmInF{A}{m}\left(i_m, :\right)} \lfro{\matri{\left( \mat{X} + \Delta \mat{X} \right)}{m}\left(i_m ,:\right) \\
    - \fmInF{A}{m}\left(i_m ,:\right) \mt{\left(\multiKhaN{A}{m}{\order}\right)}}.
\end{multline}
The problem in Eq.~\eqref{eqn:selals:nontemproal_final} is a least-square problem, and its analytical solution in Eq.~\eqref{eqn:selals:sol_nontemporal} is available.
\begin{equation}
\label{eqn:selals:sol_nontemporal}
	\fmInF{A}{m}\left(i_m, :\right) \leftarrow \matri{\left( \mat{X} + \Delta \mat{X} \right)}{m}\left(i_m, :\right) \fmInF{K}{m} \mi{\fmInF{H}{m}},
\end{equation}
where $\fmInF{K}{m} = \multiKhaN{A}{m}{\order}$ and $\fmInF{H}{m} = \multiHadaN{\mat{A}}{m}{\order}$.
\begin{comment}
    This is likewise the least square problem and we can get the solution by finding $\fmInF{A}{n}(S_{t}, \ :)$ that makes gradient zero. Here we assume that update is done in sequential order from 1 to $\order$.
    \begin{equation}
    	\fmInF{A}{n}(S_{t}, \ :) \leftarrow \matri{D(t, \tLen)}{n}(S_{t}, \ :) \mat{K}' \mat{H}'
    \end{equation}
    \begin{align}
    	\begin{split}
    	&\fmInF{A}{n}(S_{t}, \ :) \leftarrow \matri{D(t, \tLen)}{n}(S_{t}, \ :) \mat{K}' \mat{H}' \\
    	&where\ \mat{K}' = \fmInF{A}{\order + 1}_{t} \odot \fmInF{A}{\order}_{t-1} \odot \cdots \odot \fmInF{A}{n+1}_{t-1} \odot \fmInF{A}{n-1}_{t} \odot \cdots \odot \fmInF{A}{1}_{t} \\
    	& \mat{H}' = \ata{A_{t}}{1} \ast \cdots \ast \ata{A_t}{n-1} \\
    	&\qquad \ast \ata{A_{t-1}}{n+1} \ast \cdots \ast \ata{A_{t-1}}{\order} \ast \ata{A_t}{\order + 1}
    	\end{split}
    \end{align}
\end{comment}
%This approach updates only one row of $\fmInF{\mat{A}}{m}$ thus leads to the time-efficient algorithm without losing the fitting performance. 
%We show this in detail in Section \ref{sec:exp}.

%As presented in Algorithm \ref{alg:vec},
%\footnote{	\texttt{updateRowVec} is used instead of \texttt{updateRow} in Algorithm \ref{alg:total}},
After updating the $i_m$-th row of each $m$-th mode factor matrix $\fmInF{A}{m}$ either by Eq.~\eqref{eqn:selals:sol_temporal} or Eq.~\eqref{eqn:selals:sol_nontemporal}, \selals incrementally maintains $\ata{\mat{A}}{m}$ up to date by Eq.~\eqref{eqn:selals:updateA'A}.
%Given an input mode $m$ and an input index $j_m$, after the update of row $\fmInF{A}{m}(j_m, :)$, We incrementally manage 
\begin{equation} \label{eqn:selals:updateA'A}
    \ata{\mat{A}}{m} \leftarrow \ata{\mat{A}}{m} - \mt{\mat{p}}\mat{p}
                 + \mt{\fmInF{A}{m} \left(i_m, : \right)} \fmInF{A}{m} \left(i_m, : \right),
\end{equation}
where $\mat{p}$ is the $i_m$-th row of $\fmInF{A}{m}$ before the update.

\begin{algorithm}  
	\small
	\caption{\label{alg:vec} \texttt{updateRow} in \selals and \hyals}
	\tcp{Parenthesized inputs/outputs are for \hyals}
	\KwIn{(1) mode $m$ and index $i_m$ to be updated \\
		\quad\qquad(2) current tensor window  $\tensorInF{X}$, (3) change $\Delta \tensorInF{X}$, \\
		\quad\qquad (4) factor matrices $\matset{\fmInF{A}{m}}{\order}$ for $\tensorInF{X}$, \\
		\quad\qquad (5) $\matset{\ata{\mat{A}}{m}}{\order}$ (and $\matset{\mt{\fmInF{A}{m}_{prev}}\fmInF{A}{m}}{\order}$), \\
		\quad\qquad (6) (threshold $\theta$ for sampling)
		
	}
	\KwOut{updated $\fmInF{A}{m}$, $\ata{\mat{A}}{m}$ (and $\mt{\fmInF{A}{m}_{prev}}\fmInF{A}{m}$)}
	
	\SetKwFunction{updateRowRan}{updateRowRan}
	\SetKwFunction{updateRowVec}{updateRowVec}
	\SetKwFunction{sampleX}{sampleX}
	\SetKwProg{Fn}{Procedure}{:}{}
	
	\SetKwProg{Fn}{Procedure}{:}{}
	\vspace{0.5mm}
	
	\tcp{\texttt{updateRow} implemented in \selals}
	\Fn{\updateRowVec{$m, i_m$, $\cdots$}}{
		
		$\mat{p} \leftarrow \fmInF{A}{m}\left(i_m, :\right)$
		
		\textbf{if} $m = M$ \textbf{then}
			Update $\fmInF{A}{m}(i_m, :)$ by Eq.~(\ref{eqn:selals:sol_temporal}) 
		
		\textbf{else}
			Update $\fmInF{A}{m}(i_m, :)$ by Eq.~(\ref{eqn:selals:sol_nontemporal})  \label{line:vec:selals}
		
		Update $\ata{\mat{A}}{m}$ by Eq.~(\ref{eqn:selals:updateA'A}) 
		
		\Return $\fmInF{A}{m}$ and $\ata{\mat{A}}{m}$
		%, and $\mt{\fmInF{A}{m}_{prev}}\fmInF{A}{m}$
	}
	
	\tcp{\texttt{updateRow} implemented in \hyals}
	\Fn{\updateRowRan{$m$, $i_m$, $\cdots$}}{
		
		$\mat{p} \leftarrow \fmInF{A}{m} \left(i_m, : \right)$ 
		
		\If{$deg(m, i_m) \leq \theta $ \label{line:vec:selals_ran:start}}{
			Update $\fmInF{A}{m}(i_m, :)$ by Eq.~(\ref{eqn:selals:sol_nontemporal})
			\label{line:vec:selals_ran:end}
		}
		\Else{
			$S \leftarrow \theta$ indices of $\tensorInF{X}$ chosen uniformly at random, while fixing the $m$-th mode index to $i_m$ \label{line:vec:selals_ran:sample}
			
			Compute $\tensorInF{\bar{X}}$ from $S$ \label{line:vec:selals_ran:approx}
			
		%	\textbf{if} {$m = M$} \textbf{then}
		%		Update $\fmInF{A}{m}(i_m, :)$ by Eq.~(\ref{eqn:hyals:sol_first}) 
			
		%	\textbf{else}
				Update $\fmInF{A}{m}(i_m, :)$ by Eq.~(\ref{eqn:hyals:sol}) 
				\label{line:vec:hyals}
			
		}
		
		Update $\ata{\mat{A}}{m}$ by Eq.~(\ref{eqn:selals:updateA'A}) \label{line:updateAA:hyals}
		
		%\If{$m \neq \order - 1$}{
			Update $\mt{\fmInF{A}{m}_{prev}}\fmInF{A}{m}$ by Eq.~(\ref{eqn:hyals:updateA'A}) 
			
			\label{line:updateAAprev:hyals}
		%}
	
		\Return $\fmInF{A}{m}$, $\ata{\mat{A}}{m}$, and $\mt{\fmInF{A}{m}_{prev}}\fmInF{A}{m}$
	}
\end{algorithm}

\smallsection{Pros and Cons}
By updating only few rows of each factor matrix, \selals significantly reduces the computational cost of \wals, as formalized in Theorem~\ref{thm:time:selals}, without much loss in the quality of the solution. However, \selals slows down if many non-zeros are of the same index (see Eq.~\eqref{eq:time:selals}), and it is often unstable due to numerical errors, as discussed later.
%	 per event without big loss in the fitting performance.
%	But the update rate is jagged according to the degree of the index corresponding to the row.

\begin{thm}[Time complexity of \selals]
\label{thm:time:selals}
Let $deg(m,i_m)\equiv|\matri{\left( \mat{X} + \Delta \mat{X} \right)}{m}(i_m, :)|$ be the number of non-zeros of $\mat{X} + \Delta \mat{X}$ whose $m$-th mode index is $i_m$.
Then, the time complexity of \selals is 
\begin{equation}
\label{eq:time:selals}
O\bigg( \order R \sum\nolimits_{m=1}^{\order-1} deg(m, i_m) + (\order R)^2 + \order R^3 \bigg).
\end{equation}
\end{thm}
\begin{proof}
	\begin{comment}
		Computing $\matri{\left( \mat{X} + \Delta \mat{X} \right)}{m}\left(i_m, :\right)\fmInF{K}{m}$ in Eq.~\eqref{eqn:selals:sol_nontemporal} takes $O(\order\rank)$ time per non-zero in $\matri{\left( \mat{X} + \Delta \mat{X} \right)}{m}\left(i_m, :\right)$ and thus takes $O(\order\rank \cdot deg(m,i_m))$ time in total.
		%When updating the $m$-th mode matrix, it requires $O(MR\,deg(i_m))$ and it needs $O(MR \sum_{m=1}^{\order-1}deg(i_m))$ in general.
		Since  $\matset{\ata{\mat{A}}{m}}{\order}$ is maintained,
		computing $\mat{H}^{(m)} \in \dims{R}{R}$  and $\mi{\mat{H^{(m)}}}$ in Eq.~\eqref{eqn:selals:sol_nontemporal} takes $O(MR^2)$ and $O(R^3)$ time.
		Thus, the time complexity of computing Eq.~\eqref{eqn:selals:sol_nontemporal} for every $m\in \{1,\cdots,\order-1\}$ is Eq.~\eqref{eq:time:selals}, and it dominates the time complexity of the rest of \selals. 
		Hence, the time complexity of \selals is  Eq.~\eqref{eq:time:selals}.
	\end{comment}
	See Section II.B of the online appendix \cite{appendix}.
\end{proof}

\begin{comment}
\begin{thm}[Space complexity of Algorithm~\ref{alg:selals}]
The space complexity of \selals~method is $$O\left( \left(\sum_{m = 1}^{\order-1}\modeSize{m} + \tLen \right)\rank + \order\rank^2 + \order \left|\tensorWindow{D}{t}\right| \right)$$
\end{thm}

\begin{proof}
It is same as the space complexity of \wals.
\end{proof}
\end{comment}

\begin{comment}
First two terms are for factor matrices and tensor window. The final term is for $\matset{\ata{A}{n}}{\order + 1}$. There is room for reducing space consumption if total entries of tensor window is saved on disk and only the entries used in (12) are loaded to memory.
\end{comment}

\subsection{\hyalslong (\hyals)} \label{sec:algo:hybals}
\vspace{-0.5mm}

We introduce \hyals, which is even faster than \selals.
%and comparable fitting performance in this section.
The outline of \hyals (see Algorithm~\ref{alg:total}) is the same with that of \selals. 
That is, \hyals also updates only the rows of the factor matrices that approximate the changed entries in the current tensor window $\tensorInF{X}$.
However, when updating such a row, the number of entries accessed by \hyals is upper bounded by a user-specific constant $\thre$, while \selals accesses $deg(m, i_m)$ entries (see Theorem~\ref{thm:time:selals}), which can be as many as all the entries in $\tensorInF{X}$.
Below, we present its update rule. 

% in \hyals.
%
%when updating each such row, \hyals accesses a constant number of entries, while \selals accesses $deg(m, i_m)$ entries, whose 
%Notably, \hyals has \textbf{constant time complexity}, while that of \selals can be high if $deg(m, i_m)$ is huge (see Theorem~\ref{thm:time:selals})
%According to Theorem~\ref{thm:time:selals}, the time complexity of 
%\hyals ensures the worst case complexity which is proportional to a user-adjustable constant.

Assume \hyals updates the $i_m$-th row of $\fmInF{A}{m}$. That is, consider the problem in Eq.~\eqref{eqn:selals:nontemproal_final}.
As described in the procedure \texttt{updateRowRan} in Algorithm \ref{alg:vec},
\hyals uses different approaches depending on a user-specific threshold $\thre$ and $deg(m,i_m)\equiv|\matri{\left( \mat{X} + \Delta \mat{X} \right)}{m}(i_m, :)|$, i.e., the number of non-zeros of $\tensorInF{X} + \Delta \tensorInF{X}$ whose $\order$-th mode index is $i_m$.
If $deg(m,i_m)$ is smaller than or equal to $\thre$, then \hyals uses Eq.~\eqref{eqn:selals:sol_nontemporal} (lines~\ref{line:vec:selals_ran:start}-\ref{line:vec:selals_ran:end}), which is also used in \selals.

% reconsider the problem in (\ref{eqn:selals:nontemporal_origin}) of Section~\ref{sec:algo:selals:nonTemp} \change{ and its solution Eq.~(\ref{eqn:selals:sol_nontemporal})}.
%Computing $\matri{\left( \mat{X} + \Delta \mat{X} \right)}{m}(i_m, :) \multiKhaN{A}{m}{\order}$ in Eq.~(\ref{eqn:selals:sol_nontemporal}) takes a time $O(\rank\order\,deg(m,i_m))$ according to the proof of Theorem \ref{thm:selals}.
%It can be $O(\rank\order\,\thre)$ if $deg(m,i_m)$ is less than $\thre$.
%We apply Eq.~(\ref{eqn:selals:sol_nontemporal}) for updating rows which correspond to indices of which degrees are less than $\thre$.

However, if $deg(m,i_m)$ is greater than $\thre$, \hyals speeds up the update through approximation. % to set the higher bound of the computational cost. 
First, it samples $\thre$ indices from $\tensorInF{X}$ without replacement, while fixing the $\order$-th mode index to $i_m$ (line~\ref{line:vec:selals_ran:sample}).\footnote{We ignore the indices of non-zeros in $\Delta \tensorInF{X}$ even if they are sampled.} % since \hyals handle it in a different manner. 
Let the set of sampled indices be $S$, and let $\tensorInF{\bar{X}} \in \threeDimsBack{N_1}{N_{\order - 1}}{W}$ be a tensor whose entries are all 0 except those with the sampled indices $S$.
For each sampled index $J=(\indexj)\in S$, $\entryjbarshort = \entryjshort - \entryjtildeshort$.
Note that for any index $J=(\indexj)$ of $\tensorInF{X}$, $\entryjtildeshort + \entryjbarshort=\entryjshort$ if $J \in S$ and $\entryjtildeshort + \entryjbarshort=\entryjtildeshort$ otherwise.
Thus, the more samples \hyals draws with larger $S$, the closer $\reconst{\tensorInF{X}} + \tensorInF{\bar{X}}$ is to $\tensorInF{X}$.
%Assuming that \change{$\reconst{\tensorInF{X}} + \tensorInF{\bar{X}}$} approximates $\tensorInF{X}$ well, 
\hyals uses $\reconst{\tensorInF{X}} + \tensorInF{\bar{X}}$ to approximate $\tensorInF{X}$ in the update.
Specifically, it replaces $\tensorInF{X}$ of the objective function in Eq.~\eqref{eqn:selals:nontemproal_final} with $\reconst{\tensorInF{X}} + \tensorInF{\bar{X}}$. %$\reconst{\tensorInF{X}} + \tensorInF{\bar{X}} + \Delta \tensorInF{X}$ and denote it as $\hybX{\tensorInF{X}}$. That is,
%\begin{equation} \label{eqn:hyals:approx}  
%    \hybX{\tensorInF{X}} \equiv \reconst{\tensorInF{X}} + \tensorInF{\bar{X}}.
%\end{equation} 
%We use $\hat{\tensorInF{X}}$ instead of $\tensorInF{X} + \Delta \tensorInF{X}$ in the problem of  and solve the least square problem with respect to $\fmInF{A}{m}(i_m, :)$. 
Then, as in Eq.~\eqref{eqn:selals:sol_nontemporal}, the update rule in Eq.~\eqref{eqn:hyals:origin} follows.
\begin{equation} \label{eqn:hyals:origin}
    \fmInF{A}{m}(i_m, :) \leftarrow \matri{(\reconst{\tensorInF{X}} + \tensorInF{\bar{X}}+\Delta \tensorInF{X})}{m}(i_m, :)\fmInF{K}{m}\mi{\fmInF{H}{m}},
\end{equation}
where $\fmInF{K}{m} = \multiKhaN{A}{m}{\order}$ and $\fmInF{H}{m} = \multiHadaN{\mat{A}}{m}{\order}$.
Let $\fmInF{A}{m}_{prev}$ be the $m$-th mode factor matrix before the update and $\fmInF{H}{m}_{prev}$ be $\multiHadaNCustom{\mt{\fmInF{A}{n}_{prev}}\fmInF{A}{n}}{m}{\order}$.
By Eq.~\eqref{eqn:selals:KtoH}, Eq.~\eqref{eqn:hyals:origin} is equivalent to Eq.~\eqref{eqn:hyals:sol}.%\footnote{If $\order$ is the mode updated first, then $\fmInF{H}{m}_{prev} = \fmInF{H}{m}$ and $\fmInF{A}{m}(i_m, :)\fmInF{H}{m}_{prev} \mi{\fmInF{H}{m}}=\fmInF{A}{m}(i_m, :)$.}
\begin{multline} \label{eqn:hyals:sol} 
    \fmInF{A}{m}(i_m, :) \leftarrow \fmInF{A}{m}(i_m, :)\fmInF{H}{m}_{prev} \mi{\fmInF{H}{m}}+ \\
    \matri{(\mat{\bar{X}} + \Delta \mat{X})}{m}\fmInF{K}{m} \mi{\fmInF{H}{m}}.
\end{multline}
Noteworthy, $\matri{(\mat{\bar{X}} + \Delta \mat{X})}{m}$ has at most $\thre+2=O(\thre)$ non-zeros.
\hyals uses Eq.~\eqref{eqn:hyals:sol} to update the $i_m$-th row of $\fmInF{A}{m}$ (line~\ref{line:vec:hyals} of Algorithm~\ref{alg:vec}).
It incrementally maintains $\ata{\mat{A}}{m}$ up to date by Eq.~\eqref{eqn:selals:updateA'A} (line~\ref{line:updateAA:hyals}), as \selals does.
It also maintains $\mt{\fmInF{A}{m}_{prev}}\fmInF{A}{m}$ up to date by Eq.~\eqref{eqn:hyals:updateA'A}  (line~\ref{line:updateAAprev:hyals}).
\begin{equation} \label{eqn:hyals:updateA'A}
\mt{\fmInF{A}{m}_{prev}}\fmInF{A}{m} \leftarrow \mt{\fmInF{A}{m}_{prev}}\fmInF{A}{m} - \mt{\mat{p}}\mat{p} + \mt{\mat{p}}\fmInF{A}{m}(i_m, :),
\end{equation}
where $\mat{p} = \fmInF{A}{m}_{prev} \left(i_m, : \right)$. 

\smallsection{Pros and Cons}
Through approximation, \hyals upper bounds the number of non-zeros of $\matri{(\mat{\bar{X}} + \Delta \mat{X})}{m}$ in Eq.~\eqref{eqn:hyals:sol} by $O(\thre)$.
As a result, the time complexity of \hyals, given in Theorem~\ref{thm:hyals}, is lower than that of \selals.
Specifically, $deg(m,i_m)$ in Eq.~\eqref{eq:time:selals}, which can be as big as $|\tensorInF{X}+\Delta\tensorInF{X}|$, is replaced with the user-specific constant $\thre$ in Eq.~\eqref{eq:time:hyals}
Noteworthy, if we regard $\order$, $\rank$, and $\thre$ in Eq.~\eqref{eq:time:hyals} as constants, \textbf{the time complexity of \hyals becomes constant.}
This change makes \hyals significantly faster than \selals, at the expense of a slight reduction in the quality of the solution.
%Like \selals, \hyals often suffers from numerical instability, which we address in the following subsection.

%If $\order$ is the first updated mode, $\fmInF{H}{m}_{prev} = \fmInF{H}{m}$ holds so Eq.~(\ref{eqn:hyals:sol}) becomes:
%\begin{equation} \label{eqn:hyals:sol_first}
%    \fmInF{A}{m}(i_m, :) \leftarrow \fmInF{A}{m}(i_m, :) + 
%    \matri{(\mat{\bar{X}} + \Delta \mat{X})}{m}\fmInF{K}{m} \mi{\fmInF{H}{m}}
%\end{equation}

\
%Note that we use $\hybX{\tensorInF{X}}$ to update the $i_m$-th row of $\fmInF{A}{m}$. 
%Through the sampling process, we make $\thre$ entries from $\matri{\hybX{\mat{X}}}{m}(i_m, :)$ equal to $\matri{\mat{X} + \Delta \mat{X}}{m}$. 
%This helps accurate update of $\fmInF{A}{m}(i_m, :)$ as $\matri{\hybX{\mat{X}}}{m}(i_m, :)$ is used for the update process of it. 
%The function \texttt{updateRowRan} of Algorithm \ref{alg:vec} describes the update of one row using \hyals.
%Again, the function \texttt{updateRowRan} is used in place of \texttt{updateRow} in Algorithm \ref{alg:total}.

\begin{thm}[Time complexity of \hyals] \label{thm:hyals}
	If $\thre > 1$, then 
    the time complexity of \hyals is 
    \begin{equation}
    O\bigg( \order^2\rank\, \thre + \order^2\rank^2 + \order\rank^3 \bigg). \label{eq:time:hyals}
    \end{equation}
	If $\order$, $\rank$, and $\thre$ are regarded as constants,
    Eq.~\eqref{eq:time:hyals} is $O(1)$.
\end{thm}

\begin{proof}
	See Section II.C of the online appendix \cite{appendix}.
	\begin{comment}
		Computing $\matri{\left( \mat{X} + \Delta \mat{X} \right)}{m}\left(i_m, :\right)\fmInF{K}{m}$ in Eq.~\eqref{eqn:selals:sol_nontemporal} (line \ref{line:vec:selals_ran:end} of Algorithm~\ref{alg:vec})
		takes $O(\order\rank)$ time per non-zero in $\matri{\left( \mat{X} + \Delta \mat{X} \right)}{m}\left(i_m, :\right)$, and thus it takes $O(\order \rank\, \thre)$ time in total due to the condition in line~\ref{line:vec:selals_ran:start}.
		Computing $\matri{(\mat{\bar{X}} + \Delta \mat{X})}{m}\fmInF{K}{m}$ in Eq.~\eqref{eqn:hyals:sol} also takes $O(\order\rank\, \thre)$ time since $\tensorInF{\bar{X}} + \Delta \tensorInF{X}$ has at most  $(\thre + 2)$ non-zeros.
		Since  $\matset{\ata{\mat{A}}{m}}{\order}$ is maintained, computing $\mat{H}^{(m)} \in \dims{\rank}{\rank}$  and $\mi{\mat{H^{(m)}}}$ in Eq.~\eqref{eqn:selals:sol_nontemporal} and Eq.~\eqref{eqn:hyals:sol} takes $O(\order\rank^2)$ and $O(\rank^3)$ time.
		Similarly computing $\fmInF{H}{m}_{prev}$ in Eq.~\eqref{eqn:hyals:sol} takes $O(\order\rank^2)$ time.
		Thus, the time complexity of computing Eq.~\eqref{eqn:selals:sol_nontemporal} or Eq.~\eqref{eqn:hyals:sol}  for every $m\in \{1,\cdots,\order\}$ is Eq.~\eqref{eq:time:hyals}, and it dominates the time complexity of the rest of \hyals. 
		Hence, the time complexity of \hyals is  Eq.~\eqref{eq:time:hyals}.		
	\end{comment}
\end{proof}

Unlike \wals, \selals and \hyals do not normalize the columns of factor matrices during the update process. This is because normalization requires $O(\rank\sum_{m=1}^{\order}\modeSize{m})$ time, which is proportional to the number of all entries in all factor matrices, and thus significantly increases the time complexity of \selals and \hyals.
However, without normalization, the entries of factor matrices may have extremely large or extremely small absolute values, making \selals and \hyals vulnerable to numerical errors.
In our experiments (see Fig.~\ref{fig:relative_fitness} in Section~\ref{sec:exp:comparison}), the accuracies of \selals and \hyals suddenly drop due to numerical errors in some datasets.

\subsection{\ccdlong(\selccd and \hyccd)} \label{sec:algo:stable}
\vspace{-0.5mm}

In this section, we propose \selccd and \hyccd, which successfully address the aforementioned instability of \selals and \hyals.
The main idea is to clip each entry, (i.e., ensure that each entry is within a predefined range), while at the same time ensuring that the objective function does not increase.
To this end, \selccd and \hyccd employs coordinate descent, where each entry of factor matrices is updated one by one.
The outline of \selccd and \hyccd (see Algorithm~\ref{alg:total}) is the same as that of \selals and \hyals.
Below, we present their update rules, which are used in Algorithm~\ref{alg:ccd}.

%Notably, coordinate descent has been applied to scalable factorization of massive static matrices and tensors \cite{yu2012scalable,shin2016fully}.
%which have proved useful factorizing static matrices and tensors , and it 

Coordinate descent updates one variable (i.e., entry of a factor matrix) at a time while fixing all the other variables.
Assume an entry $a^{(m)}_{i_mk}$ of $\fmInF{A}{m}$ is updated. 
Solving the problem in Eq.~\eqref{eqn:obj} with respect to $a^{(m)}_{i_mk}$ while fixing the other variables is equivalent to solving the problem in Eq.~\eqref{eqn:ccd:objective}.
\begin{equation} \label{eqn:ccd:objective}
    \min_{a^{(m)}_{i_m k}} \sum_{J \in \Omega^{(m)}_{i_m}} (x_J + \Delta x_J - \sum_{r \neq k}^R \prod_{n=1}^M a_{j_n r}^{(n)} - a_{i_m k}^{(m)} \prod_{n \neq m}^M a_{j_n k}^{(n)})^2,
\end{equation}
where $J=(\indexj)$, and $\Omega^{(m)}_{i_m}$ is the set of indices of $\tensorInF{X}$ of which the $m$-th mode index is $i_m$. 
To describe its solution, we first define the following terms:
%We use the following expressions in the following paragraphs:
\begin{align}
        c_{k}^{(m)} &\equiv \prod\nolimits_{n \neq m}^M \big( \sum\nolimits_{j_n=1}^{N_n} (a_{j_n k}^{(n)} )^2 \big), \nonumber \\
        d_{i_m k}^{(m)} &\equiv \sum\nolimits_{r \neq k}^R \big( a_{i_m r}^{(m)}
        \prod\nolimits_{n \neq m}^M \big( \sum\nolimits_{j_n = 1}^{N_n}a_{j_nr}^{(n)}a_{j_nk}^{(n)} \big)\big), \label{eqn:ccd:helper} \\ 
        e_{i_m k}^{(m)} &\equiv \sum\nolimits_{r=1}^R \big( b^{(m)}_{i_m r}
        \prod\nolimits_{n \neq m}^M \big( \sum\nolimits_{j_n = 1}^{N_n}b_{j_nr}^{(n)} a_{j_nk}^{(n)} \big) \big), \nonumber
\end{align}
where $\fmInF{B}{m}\equiv \fmInF{A}{m}_{prev}$ is $\fmInF{A}{m}$ before any update.

\smallsection{Solving the Problem in Eq.~\eqref{eqn:ccd:objective}}
The problem in Eq.~\eqref{eqn:ccd:objective} is a least square problem, and thus there exists the closed-form solution, which is used to update $a_{i_m k}^{(m)}$ in Eq.~\eqref{eqn:ccd:sol_origin}.
\begin{equation} \label{eqn:ccd:sol_origin} 
a_{i_m k}^{(m)} \leftarrow \Big( \sum_{J \in \Omega^{(m)}_{i_m}} \big((x_J + \Delta x_J) \prod_{n \neq m}^M a_{j_n k}^{(n)} \big) - d_{i_m k}^{(m)}
\Big) / c_{k}^{(m)}.
\end{equation}
Eq.~\eqref{eqn:ccd:sol_origin} is used, instead of Eq.~\eqref{eqn:selals:sol_nontemporal}, when updating non-time mode factor matrices (i.e., when $m\neq \order$) in \selccd (line~\ref{line:ccd:vec} of Algorithm~\ref{alg:ccd}).
It is also used, instead of Eq.~\eqref{eqn:selals:sol_nontemporal}, in \hyccd when $deg(m,i_m) \leq \thre$  (line~\ref{line:ccd:ran:lowDeg}).
As in \selals,  when updating the time-mode factor matrix, \selccd approximates $\tensorInF{X}$ by $\reconst{\tensorInF{X}}$, and thus it uses Eq.~\eqref{eqn:ccd:sol_approx:time} (line~\ref{line:ccd:vec:time}).
\begin{equation} \label{eqn:ccd:sol_approx:time}
a_{i_m k}^{(m)} \leftarrow \Big( e_{i_m k}^{(m)}
+ \sum_{J \in \Omega^{(m)}_{i_m}} \big( \Delta x_J \prod_{n \neq m}^M a_{j_n k}^{(n)} \big) 
- d_{i_m k}^{(m)} \Big)
/ c_{k}^{(m)}.
\end{equation}
Similarly, as in \hyals, when $deg(m,i_m) > \thre$, \hyccd approximates $\tensorInF{X}$ by $\reconst{\tensorInF{X}}+\tensorInF{\bar{X}}$, and thus it uses Eq.~\eqref{eqn:ccd:sol_approx} (line~\ref{line:ccd:ran:highDeg}).
\begin{equation} \label{eqn:ccd:sol_approx}
    a_{i_m k}^{(m)} \leftarrow \Big( e_{i_m k}^{(m)}
    + \sum_{J \in \Omega^{(m)}_{i_m}} \big((\bar{x}_J+ \Delta x_J) \prod_{n \neq m}^M a_{j_n k}^{(n)} \big) 
     - d_{i_m k}^{(m)} \Big)
    / c_{k}^{(m)}.
\end{equation}
Note that all Eq.~\eqref{eqn:ccd:sol_origin}, Eq.~\eqref{eqn:ccd:sol_approx:time}, and Eq.~\eqref{eqn:ccd:sol_approx}
are based on Eq.~\eqref{eqn:ccd:helper}.
For the rapid computation of Eq.~\eqref{eqn:ccd:helper}, \selccd and \hyccd incrementally maintain $\sum_{j_m=1}^{N_m} (a_{j_mk}^{(m)})^2$ and $\sum_{j_m=1}^{N_m} a_{j_mr}^{(m)} a_{j_mk}^{(m)}$, which are the $(k, k)$-th and $(r, k)$-th entries of $\ata{\mat{A}}{m}$, by Eq.~\eqref{eqn:ccd:updatea} and Eq.~\eqref{eqn:ccd:updateaa} (lines~\ref{line:ccd:updateAA} and \ref{line:ccd:ran:updateAA}).
\hyccd also incrementally maintains $\sum_{j_m = 1}^{N_m}b_{j_mr}^{(m)} a_{j_mk}^{(m)}$, which is the $(r,k)$-th entry of $\mt{\fmInF{A}{m}_{prev}}\fmInF{A}{m}$, by Eq.~\eqref{eqn:ccd:updateba} (line~\ref{line:ccd:ran:updateAAprev}).
%Assuming $k$ is an integer from 1 to $R$ and $s \in \{1, \cdots, k-1, k+1, \cdots, R\}$, we incrementally update the terms  for the fast computation of $c$, $d$, and $e$ terms in 
%Note that both $\sum_{l=1}^{N_m} (a_{lk}^{(m)})^2$ and $\sum_{l=1}^{N_m} a_{ls}^{(m)} a_{lk}^{(m)}$ are entries of  located in , respectively.
%Also, $\sum_{l = 1}^{N_m}b_{ls}^{(m)} a_{lk}^{(m)}$ is the  in $(s, k)$.
%The update formulas are as follows:
\begin{align} 
q^{(m)}_{k k} &\leftarrow q^{(m)}_{k k} 
- (b_{i_m k}^{(m)})^2 + (a_{i_m k}^{(m)})^2, \label{eqn:ccd:updatea} \\
q^{(m)}_{r k} &\leftarrow q^{(m)}_{r k} - a_{i_m r}^{(m)} b_{i_m k}^{(m)} 
+ a_{i_m r}^{(m)} a_{i_m k}^{(m)}, \label{eqn:ccd:updateaa} \\
u^{(m)}_{r k} &\leftarrow u^{(m)}_{r k} 
- b_{i_m r}^{(m)} b_{i_m k}^{(m)} + b_{i_m r}^{(m)} a_{i_m k}^{(m)}, \label{eqn:ccd:updateba}
\end{align}
%where $p$ is the value of $a_{i_m k}^{(m)}$ before an update, 
where $\fmInF{Q}{m}\equiv\ata{\mat{A}}{m}$ and $\fmInF{U}{m}\equiv\mt{\fmInF{A}{m}_{prev}}\fmInF{A}{m}$.
Proofs of Eqs.~\eqref{eqn:ccd:sol_origin}-\eqref{eqn:ccd:updateba} can be found in an online appendix \cite{appendix}.

\smallsection{Clipping}
In order to prevent the entries of factor matrices from having extremely large or small absolute values, \selccd and \hyccd ensures that its absolute value is at most $\eta$, which is a user-specific threshold.
Specifically, if the updated entry is greater than $\eta$, it is set to $\eta$, and if the updated entry is smaller than $-\eta$, it is set to $-\eta$ (lines~\ref{line:ccd:clipping} and \ref{line:ccd:ran:clipping} in Algorithm~\ref{alg:ccd}).
Eq.~\eqref{eqn:ccd:sol_origin} followed by clipping never increases the objective function in Eq.~\eqref{eqn:ccd:objective}.
\footnote{Let $x$, $y$, and $z$ be $a^{(m)}_{i_mk}$ before update, after being updated by Eq.~\eqref{eqn:ccd:sol_origin}, and after being clipped, respectively.
The objective function in Eq.~\eqref{eqn:ccd:objective} is convex, minimized at $y$, and symmetric around $y$. $|y-z|\leq|y-x|$ holds.}${}^{,}$\footnote{For Eq.~\eqref{eqn:ccd:sol_approx:time} and Eq.~\eqref{eqn:ccd:sol_approx}, this is true only when $\tensorInF{X}$ is well approximated.}

%\texttt{updateRowVec+} and \texttt{updateRowRan+} in Algorithm \ref{alg:ccd} describe the processes of updating rows with .
%The function \texttt{clipping} prevents the target entry from becoming larger than the threshold value.

\begin{algorithm}[t]
    \small
	\caption{\label{alg:ccd} \texttt{updateRow} in \selccd and \hyccd}
	\tcp{Parenthesized inputs/outputs are for \hyccd}
	\KwIn{(1) mode $m$ and index $i_m$ to be updated, \\
		\quad\qquad (2) current tensor window  $\tensorInF{X}$, (3) change $\Delta \tensorInF{X}$, \\
		\qquad\quad(4) factor matrices $\matset{\fmInF{A}{m}}{\order}$ for $\tensorInF{X}$\\
		\qquad\quad (5) $\matset{\ata{\mat{A}}{m}}{\order}$ (and $\matset{\mt{\fmInF{A}{m}_{prev}}\fmInF{A}{m}}{\order}$) \\
		\quad\qquad (6) $\eta$ for clipping (and threshold $\theta$ for sampling)
	}
	\KwOut{updated $\fmInF{A}{m}$, $\ata{\mat{A}}{m}$ (and $\mt{\fmInF{A}{m}_{prev}}\fmInF{A}{m}$)}
	\SetKwFunction{vec}{updateRowVec+}
	\SetKwFunction{ran}{updateRowRan+}
	\SetKwFunction{clip}{clipping}
	
    \SetKwProg{Fn}{Procedure}{:}{}
    
    \vspace{0.5mm}
    %\tcp{Parenthesized inputs/outputs are for \hyccd}
    \tcp{\texttt{updateRow} implemented in \selccd}
    \Fn{\vec{$m$, $i_m$, $\cdots$}}{ 
        
        \For{$k = 1, \cdots, R$}{
 %           $p \leftarrow a_{i_m k}^{(m)}$
            
            \textbf{if} $m = M$ \textbf{then}
                Update $a_{i_m k}^{(m)}$ by Eq.~(\ref{eqn:ccd:sol_approx:time}) \label{line:ccd:vec:time} 
			
			\textbf{else} Update $a_{i_m k}^{(m)}$ by Eq.~(\ref{eqn:ccd:sol_origin}) \label{line:ccd:vec}   
			      
%            \clip{$p$, $m$, $i_m$, $k$}

			\textbf{if} $|a_{i_m k}^{(m)}| > \eta$ \textbf{then}
			$a_{i_m k}^{(m)} \leftarrow sign(a_{i_m k}^{(m)}) \cdot \eta$ \label{line:ccd:clipping} 
			
			Update $\mt{\fmInF{A}{m}}\fmInF{A}{m}$ by Eq.~(\ref{eqn:ccd:updatea}) and Eq.~(\ref{eqn:ccd:updateaa}) \label{line:ccd:updateAA}
        }
    
    	\Return  $\fmInF{A}{m}$ and $\ata{\mat{A}}{m}$
    }
    
    \tcp{\texttt{updateRow} implemented in \hyccd}
    \Fn{\ran{$m$, $i_m$, $\cdots$}}{
        
        \If{$deg(m,i_m) > \thre$}{
            $S \leftarrow \theta$ indices of $\tensorInF{X}$ chosen uniformly at random, while fixing the $m$-th mode index to $i_m$ 
	        
	        Compute $\tensorInF{\bar{X}}$ from $S$
        }
        \For{$k = 1, \cdots, R$}{
%            $p \leftarrow a_{i_m k}^{(m)}$
            
            \textbf{if} $deg(m,i_m) \leq \theta$ \textbf{then} Update $a_{i_m k}^{(m)}$ by Eq.~(\ref{eqn:ccd:sol_origin}) \label{line:ccd:ran:lowDeg}

            \textbf{else}  Update $a_{i_m k}^{(m)}$ by Eq.~(\ref{eqn:ccd:sol_approx})
             \label{line:ccd:ran:highDeg}
            
            %\clip{$p$, $m$, $i_m$, $k$}
            
            \textbf{if} $|a_{i_m k}^{(m)}| > \eta$ \textbf{then}
            $a_{i_m k}^{(m)} \leftarrow sign(a_{i_m k}^{(m)}) \cdot \eta$ \label{line:ccd:ran:clipping}   
            
            Update $\mt{\fmInF{A}{m}}\fmInF{A}{m}$ by Eq.~(\ref{eqn:ccd:updatea}) and Eq.~(\ref{eqn:ccd:updateaa}) \label{line:ccd:ran:updateAA}

            Update $\mt{\fmInF{A}{m}_{prev}}\fmInF{A}{m}$ by Eq.~(\ref{eqn:ccd:updateba}) \label{line:ccd:ran:updateAAprev}
        }
    
	    \Return $\fmInF{A}{m}$, $\ata{\mat{A}}{m}$, and $\mt{\fmInF{A}{m}_{prev}}\fmInF{A}{m}$
    }
    
%    \Fn{\clip{$p$, $m$, $i_m$, $k$}}{
%        \textbf{if} $|a_{i_m k}^{(m)}| > \eta$ \textbf{then}
%        $a_{i_m k}^{(m)} \leftarrow sign(a_{i_m k}^{(m)})\eta$
%        
%        Update $\mt{\fmInF{A}{m}}\fmInF{A}{m}$ by Eq.~(\ref{eqn:ccd:updatea}) and Eq.~(\ref{eqn:ccd:updateaa}) 
%        
%        \Return $\fmInF{A}{m}$ and $\ata{\mat{A}}{m}$
%    }
\end{algorithm}

\smallsection{Pros and Cons} 
\selccd and \hyccd does not suffer from instability due to numerical errors, which \selals and \hyals suffer from. Moreover, as shown in Theorems~\ref{thm:selccd} and \ref{thm:hyccd}, the time complexities of \selccd and \hyccd are lower than those of \selals and \hyals, respectively.
Empirically, however, \selccd and \hyccd are slightly slower and less accurate than \selals and \hyals, respectively (see Section~\ref{sec:exp:comparison}).

\begin{thm}[Time complexity of \selccd] \label{thm:selccd}
    The time complexity of \selccd is 
    \vspace{-1mm}
    \begin{equation}
    O\bigg( \order\rank\sum\nolimits_{m=1}^{\order-1} deg(m,i_m) + \order^2\rank^2 \bigg). \label{eq:time:selccd}
    \end{equation}
    \vspace{-3mm}
\end{thm}
\begin{proof}  
	See Section II.E of the online appendix \cite{appendix}.
	\iffalse
		Computing $\sum_{J \in \Omega^{(m)}_{i_m}} \big( (x_J + \Delta x_J) \prod_{n \neq m}^\order a_{j_n k}^{(n)} \big)$ takes $O(\order\cdot |\Omega^{(m)}_{i_m}|)=O(\order \cdot deg(m,i_m))$ time, and computing $c_{k}^{(m)}$ and $d_{i_m k}^{(m)}$ takes $O(\order)$ and $O(\order\rank)$ time, respectively, as $\sum_{j_n=1}^{N_n} (a_{j_n k}^{(n)})^2$ and $\sum_{j_n = 1}^{N_n}a_{j_nr}^{(n)}a_{j_nk}^{(n)}$ are maintained.
		Thus, computing Eq.~\eqref{eqn:ccd:sol_origin} takes $O(\order \cdot deg(m,i_m) + \order\rank)$ time, and
		the time complexity of computing Eq.~\eqref{eqn:ccd:sol_origin} for every $m=\{1,\cdots,\order-1\}$ and $k=\{1,\cdots,\rank\}$, which dominates the time complexity of the rest of \selccd, is Eq.~\eqref{eq:time:selccd}.
		Hence, the time complexity of \selccd is  Eq.~\eqref{eq:time:selccd}.
	\fi
\vspace{-1mm}
\end{proof}

\begin{table*}[t]
	\vspace{-5mm}
	\centering
	\caption{Summary of real-world \blue{sparse} tensor datasets. All links are at \url{https://github.com/DMLab-Tensor/SliceNStitch\#datasets}.} \label{Tab:data}
	\begin{tabular}{l|c|c|c|c}
		\toprule
		\textbf{Name} & \textbf{Description} & \textbf{Size} & \textbf{\# Non-zeros } & \textbf{Density} \\
		\midrule
		Divvy Bikes & sources $\times$ destinations $\times$ timestamps [minutes] & $673 \times 673 \times 525594$ & $3.82M$ & $1.604\times 10^{-5} $ \\
		Chicago Crime & communities $\times$ crime types $\times$ timestamps [hours] & $77\times32\times148464$ & $5.33M$ & $1.457\times10^{-2}$  \\
		New York Taxi & sources $\times$ destinations $\times$ timestamps [seconds] & $265\times265\times5184000$ & $84.39M$ & $2.318\times10^{-4}$ \\
		Ride Austin & sources $\times$ destinations $\times$ colors $\times$ timestamps [minutes] & $219\times219\times24\times285136$ & $0.89M$ & $2.739\times10^{-6}$ \\
		\bottomrule
	\end{tabular}
\end{table*}

\begin{thm}[Time complexity of \hyccd] \label{thm:hyccd}
	If $\thre > 1$, then 
	the time complexity of \hyccd is 
	\vspace{-1mm}
	\begin{equation}
	O\bigg( M^2R\,\thre + \order^2\rank^2 \bigg). \label{eq:time:hyccd}
	\end{equation}
	\vspace{-1mm}
	If $\order$, $\rank$, and $\thre$ are regarded as constants,
	Eq.~\eqref{eq:time:hyccd} is $O(1)$.
\end{thm}
\begin{proof}
	See Section II.F of the online appendix \cite{appendix}.
	\iffalse
	As explained in the proof of Theorem~\ref{thm:selccd}, computing Eq.~\eqref{eqn:ccd:sol_origin} takes $O(\order \cdot deg(m,i_m) + \order\rank)=O(\order \theta + \order\rank)$, which the equality is due to the condition in line~\ref{line:ccd:ran:lowDeg}.
	It can be shown similarly that computing Eq.~\eqref{eqn:ccd:sol_approx} takes $O(\order \theta + \order\rank)$ time.
	Note that in Eq.~\eqref{eqn:ccd:sol_approx}, $\tensorInF{\bar{X}} + \Delta \tensorInF{X}$ has at most $(\thre + 2)$ non-zeros, and computing $e_{i_m k}^{(m)}$ takes $O(\order\rank)$ time since $\sum_{j_n = 1}^{N_n} b_{j_nr}^{(n)} a_{j_nk}^{(n)}$ is maintained.
	Thus, the time complexity of computing Eq.~\eqref{eqn:ccd:sol_origin} or Eq.~\eqref{eqn:ccd:sol_approx} for every $m=\{1,\cdots,\order\}$ and $k=\{1,\cdots,\rank\}$ is Eq.~\eqref{eq:time:hyccd}.
	Since it dominates the time complexity of the rest of \hyccd, the time complexity of \hyccd is Eq.~\eqref{eq:time:hyccd}.
	\fi
\end{proof}

\section{Experiments}
\label{sec:exp}
In this section, we design and review experiments to answer the following questions:
\begin{itemize}[leftmargin=*]
	\item{\textbf{Q1. Advantages of Continuous CP Decomposition}}: What are the advantages of continuous CP decomposition over conventional CP decomposition?
    \item{\textbf{Q2. Speed and \blue{Fitness}}}: \blue{How rapidly and precisely does \method fit the input tensor, compared to baselines?}
    \item{\textbf{Q3. Data Scalability}}: How does \method scale with regard to the number of events?
    \item{\textbf{Q4. Effect of Parameters}}: How do user-specific thresholds $\theta$ and $\eta$ affect the performance of \method?
   % \item{\textbf{Q5. Effect of Clipping}}: How does clipping affect the accuracy of \method?
    \item{\textbf{Q5. Practitioner's Guide}}: Which versions of \method do we have to use?    
    \blue{\item{\textbf{Q6. Application to Anomaly Detection}}: Can \method spot abnormal events rapidly and accurately?
    }
%    \item{\textbf{Q3. Parameters}}: What is the influence of sample size, window size $W$, and period $T$ on \method?
\end{itemize}

\subsection{Experiment Specifications}
\label{sec:exp:settings}

\smallsection{Machine}
We ran all experiments on a machine with a 3.7GHz Intel i5-9600K CPU and 64GB memory.

\smallsection{Datasets}
We used four different real-world \blue{sparse} tensor datasets summarized in Table~\ref{Tab:data}.
They are sparse tensors with a time mode, and their densities vary from $10^{-2}$ to $10^{-6}$.
%Those datasets are summarized 

\smallsection{Evaluation Metrics}
We evaluated \method and baselines using the following metrics:
\begin{itemize}[leftmargin=*]

	\item{\textbf{Elapsed Time per Update}}: The average elapsed time for updating the factor matrices in response to each event.
	\item{\textbf{Fitness}} (The higher the better): \textit{Fitness} is a widely-used metric to evaluate the accuracy of tensor decomposition algorithms. It is defined as $1 - ({\lfrosmall{\reconst{\tensorInF{X}} - \tensorInF{X}}}/{\lfrosmall{\tensorInF{X}}}),$
	where $\tensorInF{X}$ is the input tensor, and $\reconst{\tensorInF{X}}$ (Eq.~\eqref{eq:reconstruct}) is its approximation.
	
    \item{\textbf{Relative Fitness} \cite{zhou2018online}} (The higher the better): \textit{Relative fitness}, which is defined as the ratio between the fitness of the target algorithm and the fitness of ALS, i.e.,
    \vspace{-0.5mm}
    \begin{equation*}
    Relative \, Fitness \equiv \frac{Fitness_{target}}{Fitness_{ALS}}.
    \vspace{-0.5mm}
    \end{equation*}
    
    Recall that ALS (see Section~\ref{sec:prelim}) is the standard batch algorithm for tensor decomposition.
  %  The more accurate a target algorithm is, the higher the relative fitness is.   
    
%    If the accuracy of the algorithm is high, then Fitness is also high and close to 1.
%    However, Fitness is not a good metric for online decomposition algorithms. This is because it depends on the ground truth data which varies a lot when the time changes.
    
%    Therefore, we alternatively used \textit{Relative Fitness}, which can reduce the data dependency.
%    It is defined
%    
%    Similar to Fitness, Relative Fitness is high when the decomposition is accurate.

\end{itemize}

\smallsection{Baselines}
Since there is no previous algorithm for continuous CPD, we compared \method with ALS, OnlineSCP \cite{zhou2018online}, \blue{CP-stream \cite{smith2018streaming}, and NeCPD ($n$)  with $n$ iterations \cite{anaissi2020necpd}, all of which are for conventional CPD (see Section~\ref{sec:related})}.
All baselines except ALS update factor matrices once per period $T$ (instead of whenever an event occurs).\footnote{\blue{We modified the baselines, which are for decomposing the entire tensor, to decompose the tensor window (see Definition~\ref{defn:tensor_window}), as \method does.}} \blue{We implemented \method and ALS in C++. 
	We used the official implementation of OnlineSCP in MATLAB and that of CP-stream in C++.\footnote{\blue{\url{https://shuozhou.github.io}, \url{https://github.com/ShadenSmith/splatt-stream}}}
	We implemented NeCPD in MATLAB.}

%\begin{itemize}[leftmargin=*]
%    \item{\textbf{ALS} (Section~\ref{sec:prelim})}: Re-running ALS from randomly initialized factor matrices whenever an event occurs.
%    \item{\textbf{OnlineSCP}} \cite{zhou2018online}: 
%    \blue{OnlineSCP is an online algorithm for CP decompositions of sparse tensors where arriving entries are stacked along the time dimension.}     
%    \blue{\item{\textbf{CP-stream}} \cite{smith2018streaming}:
%    	CP-stream is an online algorithm for CP decomposition. It down-weights past entries in its objective function.
%    	It is suitable for both sparse and dense tensors. 
%    	%with constant time and space complexities in terms of the number of previous timesteps.
%    }
%    \blue{\item{\textbf{NeCPD} ($n$)} \cite{anaissi2020necpd}: 
%    	NeCPD is an SGD-based solver that can naturally be used for online CP decomposition. Its  fitness and speed depend on the number $n$ of iterations.
%%    	We use 1 and 10 as the iteration number in this paper.
%    }
%    
%%    So OnlineSCP needs some modifications to fit our problem.
%%    \change{First, we only use the latest part of the tensor to fix the size of the tensor window.
%%    Next, we considered a single batch of size 1 as a tensor unit.
%%	Still, it can't update factor matrices every time a new change occurred.}
%\end{itemize}

%smallsection{Implementations} 

\smallsection{Experimental Setup}
We set the hyperparameters as listed in Table~\ref{Tab:data:param} unless otherwise stated.
%We set the period $T$ is set to one of the commonly used time units.
We set the threshold $\thre$ for sampling to be smaller than half of the average degree of indices (i.e., the average number of non-zeros when fixing an index) in the initial tensor window.
In each experiment, we initialized factor matrices using ALS on the initial tensor window, and we processed the events during $5\tLen \batch$ time units.
We measured relative fitness $5$ times.

%\change{The default settings of parameters for each dataset that we used in experiments are listed in .
%We fixed the rank $\rank$ to 20 and the window size $\tLen$ to 10.

%\change{ Comparisons with baselines are conducted in Section \ref{sec:exp:comparison} with the default parameters.
%These parameters are also used for checking the scalability in Section \ref{sec:exp:scalability}.
%Next, we check the effect of the threshold for degree $\thre$ in Section \ref{sec:exp:sampling}.
%Finally, we check the effect of clipping by changing the clipping value $\clipThre$ in Section \ref{sec:exp:clipping}.
%We set $\thre$ as 40 for all datasets in these experiments.
%}

\subsection{Q1. Advantages of Continuous CP Decomposition}
\vspace{-0.5mm}
\label{sec:exp:continuous}
We compared the continuous CPD and conventional CPD, in terms of the update interval (i.e., the minimum interval between two consecutive updates), fitness, and the number of parameters, using the New York Taxi dataset.
We used \hyals and fixed the period $T$ to $1$ hour for continuous CPD; and we used \blue{CP-stream, OnlineSCP, and ALS} while varying $T$ (i.e., the granularity of the time mode) from $1$ second to $1$ hour for conventional CPD.
Fig.~\ref{fig:motivation} shows the result,\footnote{\blue{Before measuring the fitness of all baselines, we merged the rows of fine-grained time-mode factor matrices sequentially by adding entries so that one row corresponds to an hour. Without this postprocessing step, the fitness of the baselines was even lower than those reported in Fig.~\ref{subfig:simple_cp:fitness}.}} and we found Observation~\ref{obs:continuous}.
\vspace{-0.5mm}
\begin{obs}[Advantages of Continuous CPD] \label{obs:continuous}
	Continuous CPD achieved (a) near-instant updates, (b) high fitness, and (c) a small number of parameters at the same time, while conventional CPD cannot. When the update interval was the same, continuous CPD achieved \blue{\textbf{$\mathbf{2.26\times}$ higher fitness with $\mathbf{55\times}$ fewer parameters} than conventional CPD}.
	When they showed similar fitness,
	the update interval of continuous CPD was \blue{\textbf{$\mathbf{3600\times}$ shorter}} than that of conventional CPD.
	\vspace{-2mm}
\end{obs} 

\begin{table}[t]
	\vspace{-0.5mm}
	\centering
	\caption{\label{Tab:data:param}Default hyperparameter settings.}
	\begin{tabular}{l|c|c|c|c|c}
		\toprule
		\textbf{Name} & \textbf{$\rank$} & \textbf{$\tLen$} & \textbf{$\batch$} (Period) & \textbf{$\thre$} & \textbf{$\eta$} \\
		\midrule
		Divvy Bikes & $20$ & $10$ & $1440 min \left(1 day\right)$ & $20$ & $1000$ \\
		Chicago Crime & $20$ & $10$ & $720 hour \left(1 month\right)$ & $20$ & $1000$  \\
		New York Taxi & $20$ & $10$ & $3600 sec \left(1 hour\right)$ & $20$ & $1000$ \\
		Ride Austin & $20$ & $10$ & $1440 min \left(1 day\right)$ & $50$ & $1000$ \\
		\bottomrule
	\end{tabular}
\end{table}

\subsection{Q2. Speed and \blue{Fitness}}
\vspace{-0.5mm}
\label{sec:exp:comparison}
\begin{figure*}[t]
	\centering
	\vspace{-4mm}
	\includegraphics[width= 0.88\linewidth]{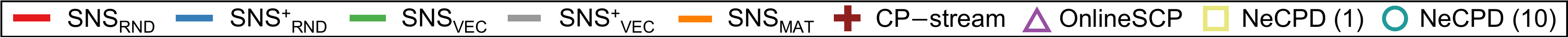}
	\\ \vspace{-1.5mm}
	\subfloat[\label{fig:relative_fitness:divvy}Divvy Bikes]{\includegraphics[width=0.5\linewidth]{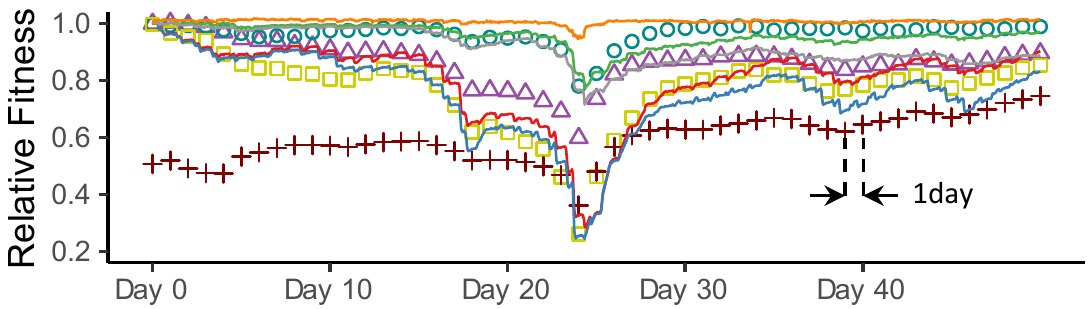}}
	\subfloat[\label{fig:relative_fitness:chicago}Chicago Crime]{\includegraphics[width=0.5\linewidth]{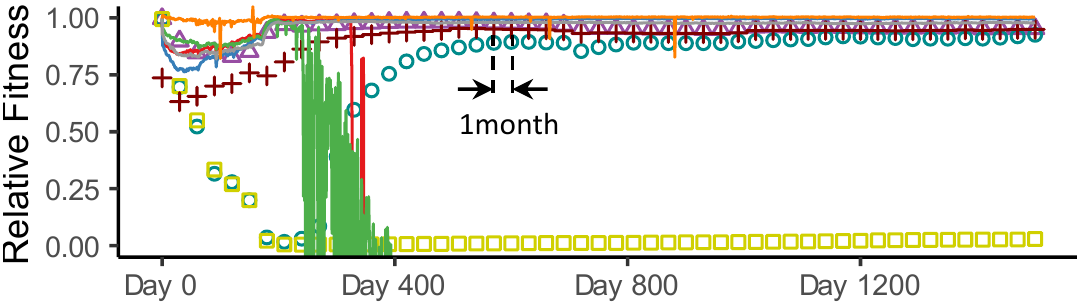}}
	\\ \vspace{-3mm}
	\subfloat[\label{fig:relative_fitness:nyt}New York Taxi]{\includegraphics[width=0.5\linewidth]{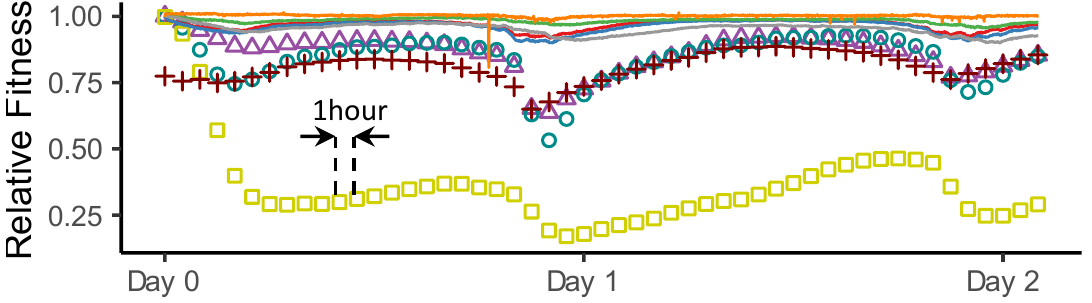}}
	\subfloat[\label{fig:relative_fitness:austin}Ride Austin]{\includegraphics[width=0.5\linewidth]{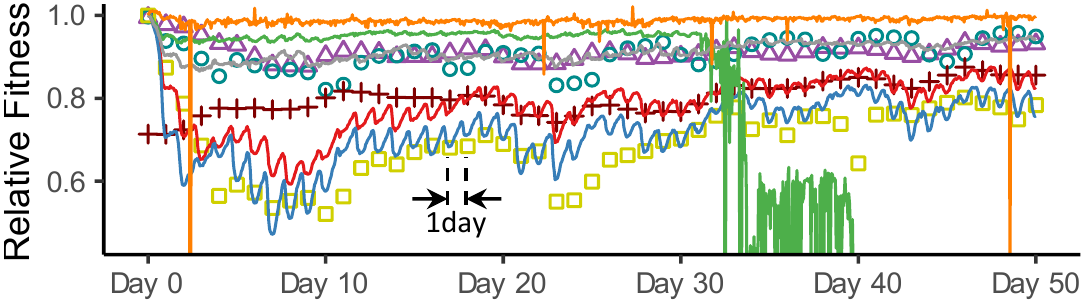}}
	\caption{\label{fig:relative_fitness} 
		\blue{Relative fitness of all versions of \method and baselines over time. 
		All versions of \method (represented as lines) update outputs whenever an event occurs, while baselines (represented as dots) update outputs only once per period $T$.}
		%\change{NeCPD (1) and NeCPD (10) are the results of running iteration once and 10 times, respectively.}			
	}
\end{figure*}

\begin{figure*} [t]
	\centering
	\vspace{-3mm}
	\includegraphics[width=0.92\linewidth]{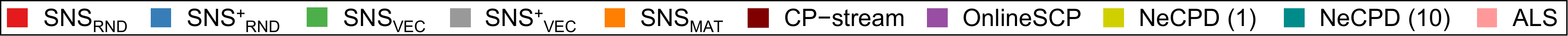} \\
	\vspace{-3.3mm}
	\subfloat[\label{fig:elapsed_time} Runtime per Update]{\includegraphics[width=0.5\linewidth,valign=b]{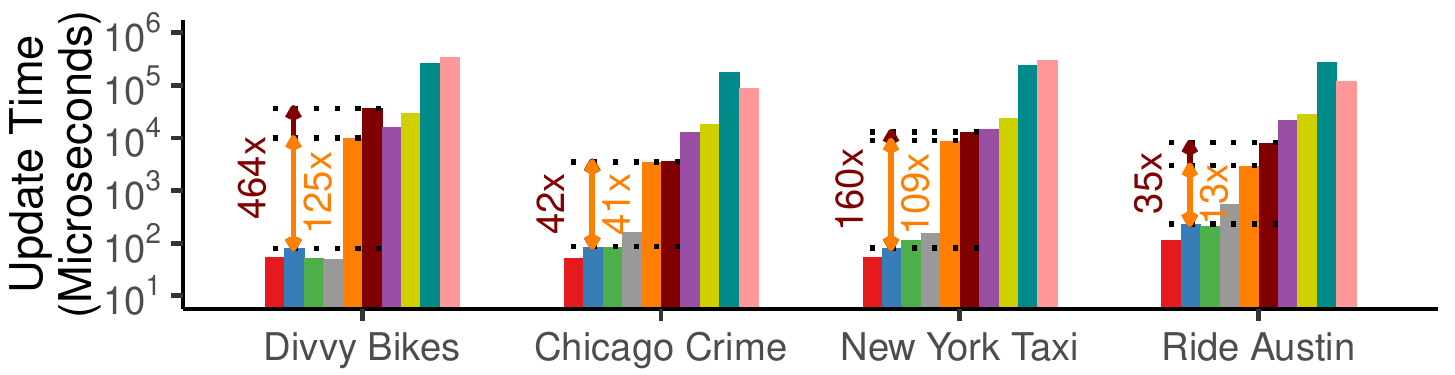}}
	\subfloat[\label{fig:relative_fitness_average}Average Relative Fitness]{\includegraphics[width=0.5\linewidth,valign=b]{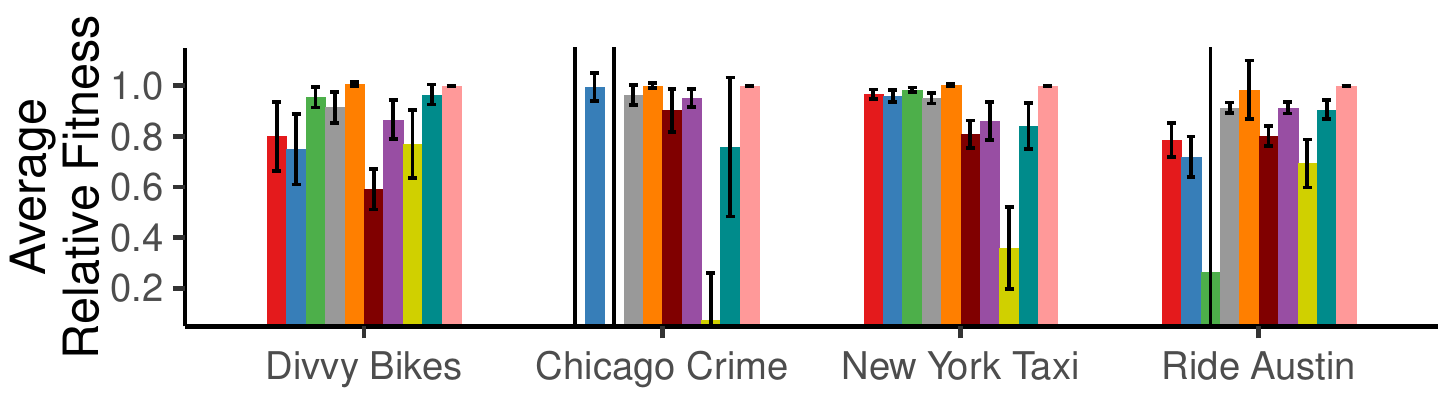}}
	\caption{\label{fig:comp} \blue{All versions of \method update factor matrices much faster with comparable \blue{fitness} than the best baseline.}}
\end{figure*}

We compared the speed and \blue{fitness} of all versions of \method and the baseline methods.
Fig.~\ref{fig:relative_fitness} shows how the relative fitness (i.e., fitness relative to ALS) changed over time, and Fig.~\ref{fig:comp} shows the average relative fitness and the average elapsed time for processing an event. We found Observations~\ref{obs:speed}, \ref{obs:errors}, and \ref{obs:accuracy}.

\begin{obs}[Significant Speed-ups] \label{obs:speed}
	All versions of \method updated factor matrices significantly faster than the fastest baseline.
	\blue{For example, \hyccd and \wals were up to \textbf{$\mathbf{464\times}$} and \textbf{$\mathbf{3.71\times}$ faster} than CP-stream, respectively}.
	\vspace{-1mm}
%		The fastest algorithm among the preferred methods by Section \ref{sec:exp:guide}, \hyccd, is up to $329$
%		times faster than \wals and $250$ times faster than OnlineSCP as seen in Fig.~\ref{fig:elapsed_time}.}
\end{obs}
\begin{obs}[Effect of Clipping] \label{obs:errors}
	\selals and \hyals failed in some datasets due to numerical errors, as discussed in the last paragraph of Section~\ref{sec:algo:hybals}. \selccd and \hyccd, where clipping is used, successfully addressed this problem.
	\vspace{-1mm}
\end{obs}
\begin{obs}[Comparable \blue{Fitness}] \label{obs:accuracy}
	All stable versions of \method (i.e., \selccd, \hyccd, and \wals) achieved $72$-$100\%$ fitness relative to the most accurate baseline. 
	%\vspace{-2mm}
%	Furthermore, our incremental methods sometimes give a much better fit than ALS method, as can bee seen 
	%Even for the Ride Austin, which is a high-dimensional sparse dataset, our methods work fine.
	%\taehyung{(Is this line necessary?)}
\end{obs}

\subsection{Q3. Data Scalability} 
\vspace{-0.5mm}
\label{sec:exp:scalability}

We measured how rapidly the total running time of different versions of \method increase with respect to the number of events in Fig.~\ref{fig:scalability}. We found Observation~\ref{obs:scalability}.
%As seen in Fig.~\ref{fig:scalability}, \hyals, \hyccd, \selals, and \selccd show the linear scalability, which are consistent to the Theorem~\ref{thm:model:complexity:time}.

\begin{obs}[Linear Scalability] \label{obs:scalability} 
The total runtime of all \method versions was linear in the number of events.
\vspace{-2mm}
\end{obs}

\begin{figure}[t]
	\centering
	\includegraphics[width=\linewidth]{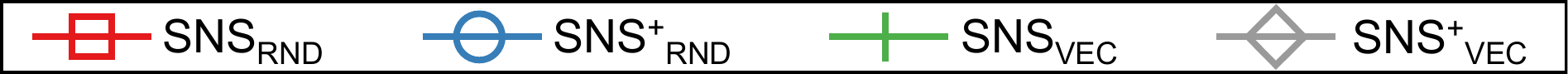} \\ \vspace{1mm}
	\includegraphics[width=\linewidth]{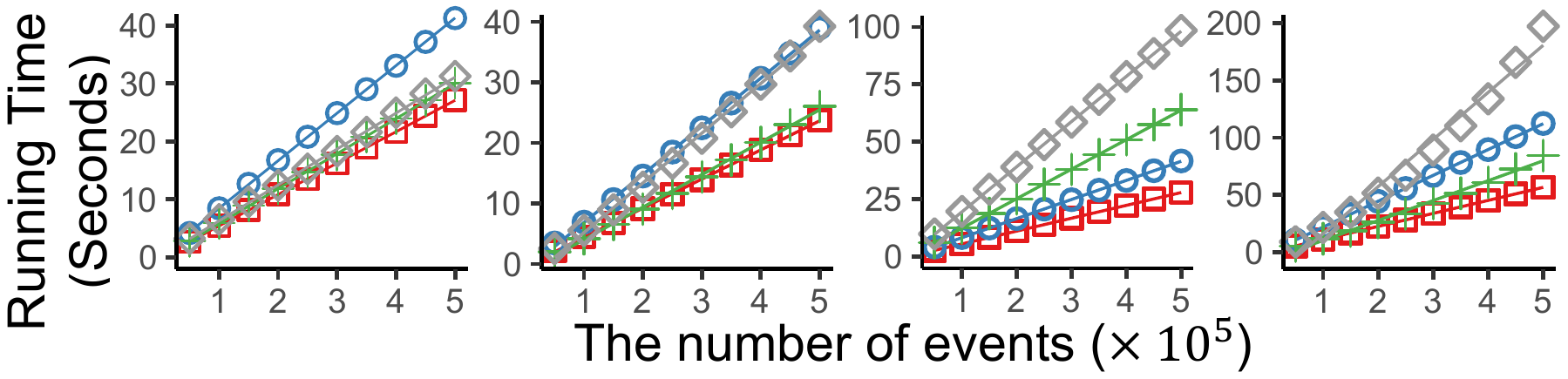}
	{\footnotesize ~~~~~(a) Divvy Bikes (b) Chicago Crime (c) New York Taxi (d) Ride Austin}
	\caption{\label{fig:scalability} The total runtime of \method is linear in the number of events. While we omit \wals due to long execution time, it shows a similar trend.
	}
\end{figure}

\subsection{Q4. Effect of Parameters}
\vspace{-0.5mm}
\label{sec:exp:sampling}
To investigate the effect of the threshold $\thre$ for sampling
on the performance of \hyals and \hyccd, we measured their relative fitness and update time while varying $\thre$ from $25\%$ to $200\%$ of the default value in Table~\ref{Tab:data:param}.\footnote{We set $\clipThre$ to $500$ in the Chicago Crime dataset since setting it to $1000$ led to unstable results.} 
The results are reported in Fig.~\ref{fig:sample}, and we found Observation~\ref{obs:thre}.
%\blue{For the effect of the threshold $\eta$ for clipping, see Section III.A of \cite{appendix}.}

%\change{We changed the threshold for degree $\thre$ by 0.25, 0.5, 1, and 2 times from the default parameters for all datasets except Chicago Crime.
%Chicago Crime shows instability in terms of relative fitness when using 1000 as $\clipThre$ so we set $\clipThre$ as 500 only for this dataset.
%We check the relative fitness and elapsed time per update for all settings.

%\change{As can be seen in Fig.~\ref{fig:sample}, There is a trade-off between relative fitness and speed depending on $\thre$.
%Relative Fitness increases as $\thre$ increases but the slope of the plot gradually decreases.
%Simultaneously, elapsed time per update also increases in linear, as we expected in time complexity analysis.}
%In general, \hyals is faster and has better accuracy than \hyccd, but it is unstable so that it cannot be drawn in the Fig.~when we use Chicago Crime dataset.

\begin{figure}[t]
	%\vspace{-5mm}
	\centering
	\includegraphics[width= 0.48\linewidth]{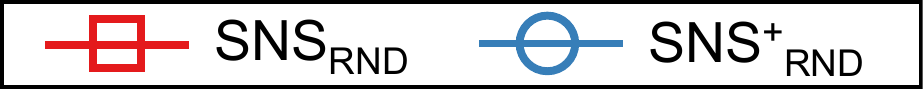} \\ 
	\textbf{\small Effect on Relative Fitness}: \hfill \ ~ \\
	\vspace{1mm}
	\includegraphics[width=\linewidth]{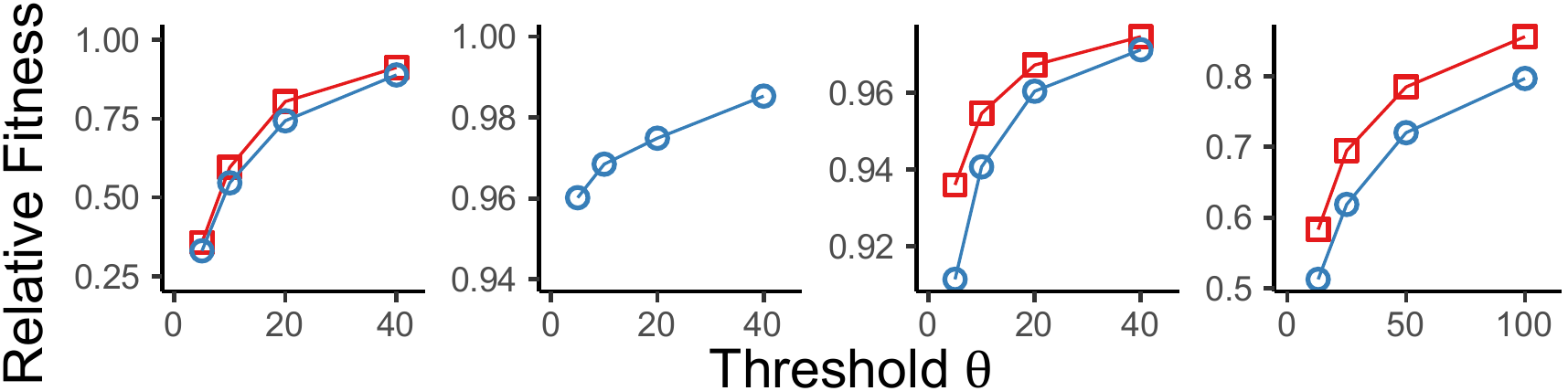} \\
	\vspace{-1mm}
	{\footnotesize ~~~~~(a) Divvy Bikes (b) Chicago Crime (c) New York Taxi (d) Ride Austin}
	\\ \vspace{1mm}
	\textbf{\small Effect on Speed (Elapsed Time per Update)}: \hfill \ ~ \\
	\vspace{1mm}
	\includegraphics[width=\linewidth]{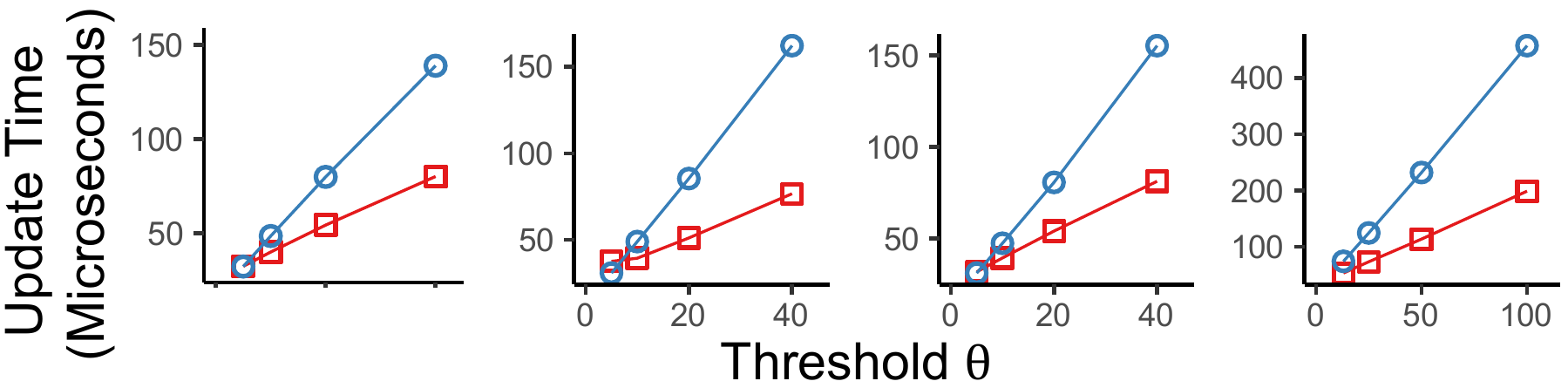} \\
	\vspace{-1mm}
	{\footnotesize ~~~~~(e) Divvy Bikes (f) Chicago Crime (g) New York Taxi (h) Ride Austin}
	\caption{\label{fig:sample}Effect of $\thre$ on the performance of \hyals and \hyccd.
		As $\thre$ increases, the fitness increases with diminishing returns, while the runtime grows linearly.
		\hyals fails in the Chicago Crime dataset due to instability.
	}
\end{figure}

\begin{obs}[Effect of $\thre$] \label{obs:thre}
As $\thre$ increases (i.e., more indices are sampled), the fitness of \hyals and \hyccd increases with diminishing returns, while their runtime grows linearly. 
%The sampling idea makes \method faster at the expense of fitness.
%\change{As we sample more, the performance increases with the gradual stuck, but the elapsed time linearly increases.}
%Fewer samples make the algorithm even faster at the expense of fitness.
%\taehyung{Emphasize the linear growth of time, on the other hand performance stuck??}
\vspace{-1mm}
\end{obs}

In Fig.~\ref{fig:clipping}, we measured the effect of the threshold $\eta$ for clipping on the relative fitness of \selccd and \hyccd, while changing $\eta$ from $32$ to $16,000$. Note that $\eta$ does not affect their speed.
We found Observation~\ref{obs:eta}.

\begin{obs}[Effect of $\eta$] \label{obs:eta}
	The fitness of \selccd and \hyccd is insensitive to $\eta$ as long as $\eta$ is small enough.
	%	\method algorithms with clipping (\selccd and \hyccd) are more stable than \selals and \hyals.
	%	\change{Their performances do not drop rapidly for all 4 datasets if we choose clipping values properly.}
	%\vspace{-2mm}
\end{obs}

\begin{figure}[h]
	\vspace{-4mm}
	\centering
	\includegraphics[width= 0.45\linewidth]{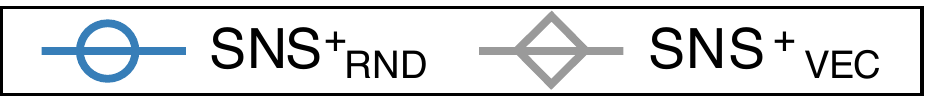} \\
	\vspace{1mm}
	\includegraphics[width=\linewidth]{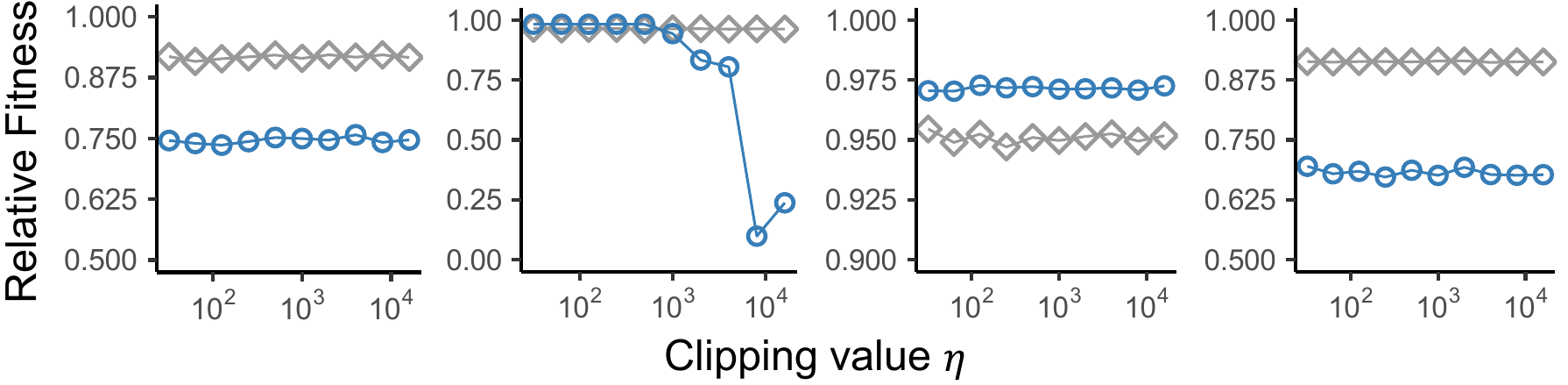} \\
	\vspace{-1mm}
	{\footnotesize ~~~~~(a) Divvy Bikes (b) Chicago Crime (c) New York Taxi (d) Ride Austin}
	\caption{\label{fig:clipping} Effects of $\eta$ on the \blue{fitness} of \selccd and \hyccd. The \blue{fitness} is insensitive to $\eta$, as long as it is small enough.}
\end{figure}

%\subsection{Q5. Effect of Clipping} \label{sec:exp:clipping}

%As discussed in Section~\ref{sec:algo:stable} and seen in Figures~\ref{fig:relative_fitness} and \ref{fig:comp}, clipping each parameter to ensure that it has absolute value at most $\eta$ in \selccd and \hyccd prevents instability of \selals and \hyals due to numerical errors.
%The instability of the algorithms can be reduced by applying the clipping method.
%As seen in Fig.~\ref{fig:relative_fitness:chicago} and \ref{fig:relative_fitness:austin}, there is no unstable behavior when we use \selccd and \hyccd.
%This is because the clipping inhibits the occurrence of large values in factor matrices.
%\change{In Fig.~\ref{fig:clipping}, we can see that the clipping value $\eta$ generally does not affect the accuracy of the algorithm.
%However, if $\eta$ becomes large then clipping can be meaningless like the case of Chicago crime.}

\subsection{Q5. Practitioner's Guide}
\vspace{-0.5mm}
\label{sec:exp:guide}
Based on the above theoretical and empirical results, we provide a practitioner's guide for \method's users. 
\begin{itemize}[leftmargin=*]
	\item We do not recommend \selals and \hyals. They are prone to numerical errors and thus unstable.
	\item We recommend using the version \blue{fitting the input tensor best} within your runtime budget among \wals, \selccd, and \hyccd. There exists a clear trade-off between their speed and \blue{fitness}. In terms of speed, \hyccd is the best followed by \selccd, and \wals is the slowest. In terms of \blue{fitness}, \wals is the best followed by \selccd, and \hyccd is the most inaccurate.
	\item If \hyccd is chosen, we recommend increasing $\thre$ as much as possible, within your runtime budget. 
\end{itemize}

\subsection{\blue{Q6. Application to Anomaly Detection}}
\vspace{-0.5mm}
\label{sec:exp:application}
\blue{We applied \hyccd, OnlineSCP, and CP-stream to an anomaly detection task.
	In the New York Taxi dataset, we injected abnormally large changes (specifically, $15$, which is $5$ times the maximum change in 1 second in the data stream) in $20$ randomly chosen entries. 
	Then, we measured the Z scores of the errors in all entries in the latest tensor unit, where new changes arrive, as each method proceeded.
	After that, we investigated the top-$20$ Z-scores in each method. As summarized in Fig.~\ref{fig:anomaly}, the precision, which is the same as the recall in our setting, was highest in \hyccd and OnlineSCP.
	More importantly, the average time gap between the occurrence and detection of the injected anomalies was \blue{about $0.0015$ seconds} in \hyccd, while that exceeded $1,400$ seconds in the others, which have to wait until  the current period ends to update CPD.	
}
%	It detects 16 among 20 anomalies and this result is the same as or better than baselines.	
%	Moreover, our methods are possible to alert anomalous events as soon as they arrive (see Table~\ref{Tab:anomaly}).
%	Other methods have to wait until the end of the current period due to their mechanisms that update CP decompostion once a period.

%\change{We focus on a real time anomaly detection for checking the performance of \method in a data analysis task.
%	A major festival can cause heavy traffics from a specific source to a specific destination (e.g. a baseball game increases traffics from a train station to a baseball park).
%	This leads to a situation where anomalous events with large values are included in a data stream.
%	In New York Taxi dataset, We randomly inject $20$ events whose values are 15 (5 times of the max value in the original data stream) and check the Z-scores of errors for entries in the latest tensor unit where new changes arrive
%}

%\change{
%	As seen in Fig.~\ref{subfig:anomaly:ours}, \method clearly shows the injected anomalies.
%	It detects 16 among 20 anomalies and this result is the same as or better than baselines.	
%	Moreover, our methods are possible to alert anomalous events as soon as they arrive (see Table~\ref{Tab:anomaly}).
%	Other methods have to wait until the end of the current period due to their mechanisms that update CP decompostion once a period.
%}
	
%\begin{practice} 
%\end{practice}

\section{Related Work}
\label{sec:related}
In this section, we  review related work on online CP decomposition (CPD) and window-based tensor analysis.
Then, we briefly discuss the relation between CPD and machine learning.
See \cite{kolda2009tensor,papalexakis2016tensors} for more models, solvers, and applications.

%\kijung{from here}

\subsection{\blue{Online Tensor Decomposition}}
\vspace{-0.5mm}
\label{sec:related_work:onlinecp}

Nion et al. \cite{nion2009adaptive} proposed Simultaneously Diagonalization Tracking (SDT) and Recursively Least Squares Tracking (RLST) for incremental CP decomposition (CPD) of three-mode dense tensors.
Specifically,
SDT incrementally tracks the SVD of the unfolded tensor, while RLST recursively updates the factor matrices to minimize the weighted squared error.
A limitation of the algorithms is that they are only applicable to three-mode dense tensors.
%The main issue of these algorithms is that they are only limited to decompose three-mode dense tensors.
Gujral et al. \cite{gujral2018sambaten} proposed a sampling-based method called SamBaTen for incremental CPD of three-mode dense and sparse tensors.
Zhou et al. proposed onlineCP \cite{zhou2016accelerating} and onlineSCP \cite{zhou2018online} for incremental CPD of higher-order dense tensors and sparse tensors, respectively.
\blue{
	Smith et al. \cite{smith2018streaming} proposed an online CPD algorithm that can be extended to include non-negativity and sparsity constraints.
	It is suitable for both sparse and dense tensors.
	SGD-based methods \cite{anaissi2020necpd,ye2019online} have also been developed for online CPD.
	Specifically, Ye et al.~\cite{ye2019online} proposed one for Poisson-distributed count data with missing values and Anaissi et al.~\cite{anaissi2020necpd} employed Nesterov's Accelerated Gradient method into SGD.
	Sobral et al.~\cite{sobral2015online} developed an online framework for subtracting background pixels from multispectral video data.}
However, all these algorithms process every new entry with the same time-mode index (e.g., a slice in Fig.~\ref{subfig:coarse}) at once.
They are not applicable to continuous CPD (Problem~\ref{defn:problem:cpd}), where changes in entries need to be processed instantly.

\blue{
BICP~\cite{huang2016bicp} efficiently updates block-based CPD~\cite{phan2011parafac, li20162pcp} when the size of the input tensor is fixed but some existing entries change per update.
It requires partitioned subtensors and their CPDs rather than the CPD of the entire input tensor.}
%Our methods differ from BICP in that they don't use an additional deep-level CP decomposition,  CP decomposition~\cite{li20162pcp}.

Moreover, several algorithms have been developed for incrementally updating the outputs of matrix factorization \cite{he2016fast,devooght2015dynamic}, Bayesian probabilistic CP factorization \cite{du2018probabilistic}, and generalized CPD \cite{zhou2017sced}, when previously unknown entries are revealed afterward.
They are also not applicable to continuous CPD (Problem~\ref{defn:problem:cpd}), where even increments and decrements of revealed entries need to be handled (see Definition~\ref{defn:delta}).

\begin{figure}[t]
	\centering
	\vspace{-5mm}
	\includegraphics[width=0.7\linewidth]{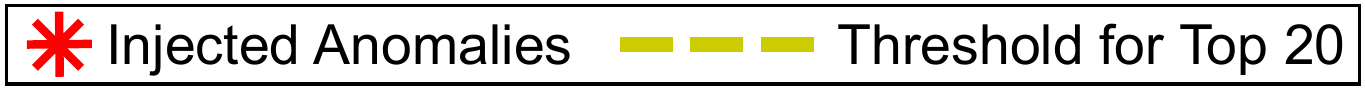} \\ \vspace{-3mm}
	\subfloat[\blue{Z-scores in \hyccd (left), OnlineSCP (middle), and CP-stream (right)}]{\includegraphics[width=1\linewidth]{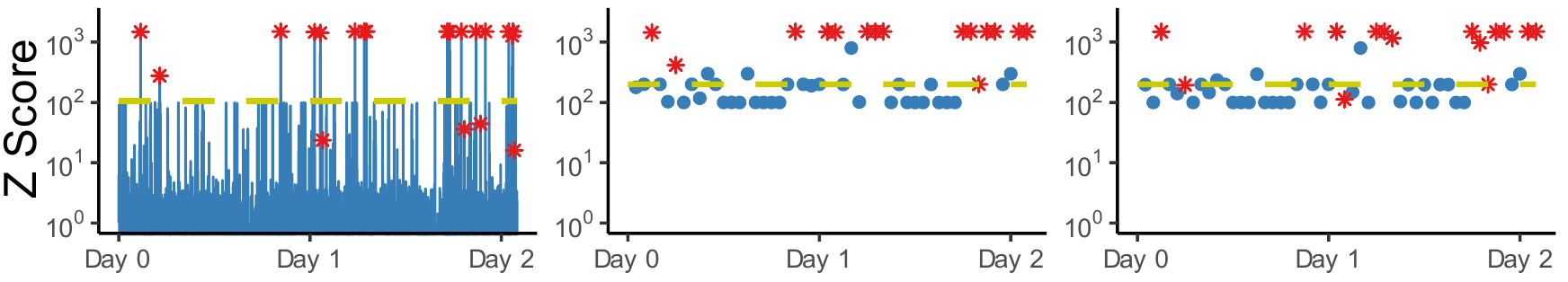}\label{subfig:anomaly:ours}} \\  \vspace{-3mm}
	%	\subfloat[\change{Anomaly scores of onlineSCP}]{\includegraphics[width=0.5\linewidth]{FIG/anomaly_onlineSCP.pdf}\label{subfig:anomaly:onlinescp}}	
	%	\subfloat[\change{Anomaly scores of CP-stream}]{\includegraphics[width=0.5\linewidth]{FIG/anomaly_cpstream.pdf}\label{subfig:anomaly:cpstram}} \\ \vspace{-2mm}	
	\subfloat[\blue{Numerical Comparison with Baselines}]{
		\scalebox{0.68}{
			\begin{tabular}{l|c|c}
				\toprule
				& Precision @ Top-20 &  Time Gap between Occurrence and Detection \\ 
				\midrule
				\hyccd & \textbf{0.80} & \textbf{0.0015 seconds} \\
				\midrule
				OnlineSCP & \textbf{0.80} &  1601.00 seconds\\
				\midrule
				CP-stream & 0.70 &  1424.57 seconds\\
				\bottomrule
			\end{tabular}
		}
		\label{Tab:anomaly}
	}   \\
	\caption{\blue{\label{fig:anomaly}\method (spec., \hyccd) detects injected anomalies in the New York Taxi dataset much faster with comparable accuracy than baselines.}
	}
\end{figure}

\blue{
Lastly, it is worth mentioning that there have been several studies on approximation properties of some offline CPD algorithms.
Haupt et al. \cite{haupt2017near} proved a sufficient condition for a sparse random projection technique to solve the low-rank tensor regression problem efficiently with an approximation quality.
%, which is a superset of the CP decomposition problem,
Song et al. \cite{song2016sublinear} showed that an importance-sampling based orthogonal tensor decomposition algorithm achieves a sublinear time complexity with provable guarantees. %under certain conditions 
To the best of our knowledge, however, there has been limited work on theoretical properties of online CPD of tensor streams.
%Despite this interest, there are relatively few studies on the approximation properties of the online tensor decomposition algorithms that focus on the dynamically updated tensors, to the best of our knowledge.
%This field of research will need to be studied more in the future.
}

% it is difficult to use them for elementwise real-time updates.
%Lastly, there are some researches about elementwise dynamic learning for matrix and tensor.
%Devooght et al. \cite{devooght2015dynamic} used block coordinate descent method, and He et al. \cite{he2016fast} proposed elementwise ALS algorithm for matrix factorization in the recommendation problem.
%Both of them only update rows in factor matrices correspond to the new entry.
%After that, Zhou et al. \cite{zhou2017sced} generalized elementwise ALS algorithm for sparse tensor decomposition with constraints.
%\iffalse
%\begin{figure}[t]
%	%\vspace{-5mm}
%	\centering
%	\includegraphics[width= 0.45\linewidth]{FIG/clipping_legend} \\
%	\vspace{1mm}
%	\includegraphics[width=\linewidth]{FIG/clipping.pdf} \\
%	\vspace{-1mm}
%	{\footnotesize ~~~~~(a) Divvy Bikes (b) Chicago Crime (c) New York Taxi (d) Ride Austin}
%	\caption{\label{fig:clipping} Effects of $\eta$ on the accuracy of \selccd and \hyccd. The accuracy is insensitive to $\eta$, as long as it is small enough.}
%\end{figure}
%\fi

\subsection{Window-based Tensor Analysis}
\vspace{-0.5mm}
Sun et al. \cite{sun2006window, sun2008incremental} first suggested the concept of window-based tensor analysis (WTA). Instead of analyzing the entire tensor at once, they proposed to analyze a temporally adjacent subtensor within a time window at a time, while sliding the window.
%In order to reduce computational cost, they proposed to analyze , instead of the entire tensor, at a time.
Based on the sliding window model, they devised an incremental Tucker decomposition algorithm for tensors growing over time.
%There have been several approaches based on the sliding window model.
Xu et al. \cite{xu2019anomaly} also suggested a Tucker decomposition algorithm for sliding window tensors and used it to detect anomalies in road networks.
Zhang et al. \cite{zhang2018variational} used the sliding window model with exponential weighting for robust Bayesian probabilistic CP factorization and completion.
Note that all these studies assume a time window moves `discretely', while in our continuous tensor model, a time window moves `continuously', as explained in Section~\ref{sec:model}.

%\begin{figure*}[t]
%	\vspace{-5mm}
%\end{figure*}

\subsection{Relation to Machine Learning}
\vspace{-0.5mm}
\label{sec:related_work:ml}

\blue{CP decomposition (CPD) has been a core building block of numerous machine learning (ML) algorithms, which are designed for classification \cite{rendle2010factorization}, weather forecast \cite{xu2018muscat}, recommendation \cite{yao2015context}, stock price prediction \cite{spelta2017financial},  to name a few. Moreover, CPD has proven useful for outlier removal \cite{najafi2019outlier,lee2021robust}, imputation \cite{shin2016fully,lee2021robust}, and dimensionality reduction \cite{kolda2009tensor}, and thus it can be used as a preprocessing step of ML algorithms, many of which are known to be vulnerable to outliers, missings, and the curse of dimensionality. 
We refer the reader to \cite{sidiropoulos2017tensor} for more roles of tensor decomposition for ML.
By making the core building block ``real-time'', our work represents a step towards real-time ML.
Moreover, \method can directly be used as a preprocessing step of existing streaming ML algorithms.
}

%\change{
%Tensor decomposition and the machine learning are complementary to each other.
%Typically, tensor decomposition can be used as a data preprocessing or a building block for the machine learning algorithm.
%Representatively, it can be used in the recommendation system \cite{yao2015context,shin2016fully} for higher-order data, and it also plays an important role in performing dimensionality reduction inside a higher-order PCA \cite{kolda2009tensor}.\\
%%Therefore, it is unfair to compare tensor decomposition and the machine learning algorithms directly.
%Additionally, tensor decomposition can be said to be more robust than machine learning algorithms.
%Many existing machine learning-based algorithms including the neural network model are known to be difficult to handle outliers \cite{goodfellow2014explaining, szegedy2013intriguing} or missing values.
%On the other hand, tensor decomposition-based algorithms can handle both of them due to its low-rank property \cite{zhang2018variational, najafi2019outlier}.
%}

\section{Conclusion}
\label{sec:conclusion}
In this work, we propose \method, % for ``real-time'' analysis of \blue{sparse} tensors, 
aiming to make tensor analysis ``real-time'' and applicable to time-critical applications. %, such as
%anomaly detection \cite{koutra2012tensorsplat}, stock price prediction \cite{spelta2017financial},  recommendation \cite{yao2015context}, and weather forecast \cite{xu2018muscat}.
%the data model and optimization algorithms that can update tensor and its decomposition in an event-driven manner from a data stream.
We summarize our contributions as follows:
%In this work, we propose \method for ``real-time'' analysis of \blue{sparse} tensors. It makes tensor analysis applicable to time-critical applications, such as
%anomaly detection \cite{koutra2012tensorsplat}, stock price prediction \cite{spelta2017financial},  recommendation \cite{yao2015context}, and weather forecast \cite{xu2018muscat}.
%%the data model and optimization algorithms that can update tensor and its decomposition in an event-driven manner from a data stream.
%We summarize our contributions as follows:
\begin{itemize}[leftmargin=*]
    \item{\textbf{New data model}}: We propose the continuous tensor model and its efficient event-driven implementation  (Section~\ref{sec:model}).
    With our CPD algorithms, it achieves near real-time updates, high \blue{fitness}, and a small number of parameters (Fig.~\ref{fig:motivation}).
    \item{\textbf{Fast online algorithms}}: 
    We propose a family of online algorithms for CPD in the continuous tensor model (Section~\ref{sec:algo}).
    They update factor matrices in response to changes in an entry up to \blue{$464\times$} faster than online competitors, with \blue{fitness} even comparable (spec., $72$-$100\%$) to offline competitors (Fig.~\ref{fig:comp}).
    We analyze their complexities (Theorems~\ref{thm:time:wals}-\ref{thm:hyccd}).
%    For example, \hyccd is up to $329\times$ faster than offline algorithms
%    Our vectorwise update solution only requires constant time($\leq$ 0.6 millisecond) per event. 
 %   The fastest algorithm \hyals is up to three orders of magnitude faster than the ``workhorse" algorithm, ALS .
    \item{\textbf{Extensive experiments}}: We evaluate the speed, \blue{fitness}, and scalability of our algorithms on $4$ real-world \blue{sparse} tensors. We analyze the effects of hyperparameters. The results indicate a clear trade-off between speed and \blue{fitness}, based on which we provide practitioner's guides (Section~\ref{sec:exp}).
 %    effect of clipping in \selccd and \hyccd, and sampling in \hyals and \hyccd are also investigated (Section \ref{sec:exp:clipping} and Section \ref{sec:exp:sampling}). 

\end{itemize}

\textbf{Reproducibility}: The code and datasets used in the paper are available at \url{https://github.com/DMLab-Tensor/SliceNStitch}.

\section*{Acknowledgement}
\textls[-15]{\small This work was supported by Samsung Electronics Co., Ltd., Disaster-Safety Platform Technology Development Program
of the National Research Foundation of Korea (NRF) funded by the Ministry of Science and ICT (No.
2019M3D7A1094364), and Institute of Information \& Communications Technology Planning \& Evaluation (IITP) grant funded by the Korea government (MSIT) (No. 2019-0-00075, Artificial Intelligence Graduate School Program (KAIST)).}

\bibliographystyle{IEEEtran}
\bibliography{BIB/reference}

\end{document}